\documentclass[9pt,letterpaper,twoside,reqno,nosumlimits]{amsart}

\usepackage{etoolbox}

\patchcmd{\section}{\scshape}{\bfseries}{}{}
\makeatletter
\renewcommand{\@secnumfont}{\bfseries}
\makeatother

\usepackage{comment,color}
\usepackage[dvipsnames]{xcolor}
\patchcmd{\section}{\normalfont}{\normalfont\color{MidnightBlue}}{}{}
\patchcmd{\subsection}{\normalfont}{\normalfont\color{MidnightBlue}}{}{}

\makeatletter
\def\subsubsection{\@startsection{subsubsection}{3}%
\z@{.5\linespacing\@plus.7\linespacing}{-.5em}%
{\normalfont\bfseries}}
\makeatother

\usepackage{circuitikz}

\usepackage{chemformula}

\usepackage{fancyhdr}
\usepackage{amsmath,amsfonts,amsbsy,amsgen,amscd,mathrsfs,amssymb,amsthm,mathtools,tensor}

\usepackage[a4paper]{geometry}
\usepackage{tikz}
\usepackage{overpic}
\usepackage{dirtytalk}

\usepackage{stmaryrd}
\usepackage{algorithm}
\usepackage{algpseudocode}

\usepackage{amsopn}

\usepackage{caption}
\usepackage{subcaption}

\usepackage{graphicx}
\usepackage{color}

\usepackage{url}
\usepackage{epstopdf}
\usepackage{pdfsync}
\usepackage[colorlinks]{hyperref}

\usepackage{comment}
\usepackage{mathabx}
\usepackage{upgreek}

\usepackage{mathrsfs}
\usepackage{yfonts}

\usepackage{amsaddr}
\hypersetup{
    linkcolor=blue,
}

\newlength{\fixboxwidth}
\usepackage{epsfig,amsbsy,graphicx,multirow}

\renewcommand{\algorithmiccomment}[1]{\bgroup\hfill//~#1\egroup}

\usepackage[all,cmtip]{xy}

\usepackage{accents}

 \numberwithin{equation}{section}

\def\V{\mathcal{V}}

\def\P{\mathbb{P}}
\def\R{\mathbb{R}}

\def\cN{\mathcal{N}}

\def\S{\mathcal{S}}

\def\Z{\mathcal{Z}}
\def\E{\mathbb{E}}

\def\M{\mathcal{M}}

\def\X{{\bf\mathcal{X}}}
\def\F{\mathcal{F}}

\def\H{\mathcal{H}}

\def\N{\mathcal{N}}

\def\restrict#1{\raise-.5ex\hbox{\ensuremath|}_{#1}}

\def\<{\big\langle}
\def\>{\big\rangle}

\def\Var{\operatorname{Var}}

\def\Span{\operatorname{span}}

\definecolor{red}{rgb}{0.9, 0, 0}

\definecolor{green}{rgb}{0.0, 1.0, 0.0}

\newtheorem{Theorem}{Theorem}
\newtheorem{Proposition}{Proposition}

\newtheorem{Remark}{Remark}
\newtheorem{Example}{Example}

\newtheorem{Problem}{Problem}

\newcommand\blfootnote[1]{%
  \begingroup
  \renewcommand\thefootnote{}\footnote{#1}%
  \addtocounter{footnote}{-1}%
  \endgroup
}

\makeatletter
\newcommand{\oset}[3][0ex]{%
  \mathrel{\mathop{#3}\limits^{
    \vbox to#1{\kern-2\ex@
    \hbox{$\scriptstyle#2$}\vss}}}}
\makeatother

\makeatletter
\newcommand{\uset}[3][0ex]{%
  \mathrel{\mathop{#3}\limits_{
    \vbox to#1{\kern-2\ex@
    \hbox{$\scriptstyle#2$}\vss}}}}
\makeatother

\usepackage{aurical}
\usepackage[T1]{fontenc}

\usepackage{tikz,pgflibraryshapes}
\usetikzlibrary{shadows}
\usetikzlibrary{arrows}

\usepackage{natbib}

\usepackage{etoolbox}
\preto{\section}{\counterwithout{equation}{section}}

\begin{document}

\title[]{Co-discovering Graphical Structure and Functional Relationships Within Data: A Gaussian Process Framework for Connecting the Dots}

\author[Bourdais \and  Batlle \and  Yang \and  Baptista \and  Rouquette \and  Owhadi]{Th\'{e}o Bourdais$^\dagger$, Pau Batlle$^\dagger$, Xianjin Yang$^\dagger$, Ricardo Baptista$^\dagger$,\\ Nicolas Rouquette$^*$ \and Houman Owhadi$^\dagger$}
\blfootnote{$^\dagger$Computing and Mathematical Sciences, California Institute of Technology, Pasadena, CA 91125, USA} 
\blfootnote{$^*$Jet Propulsion Laboratory, California Institute of Technology, Pasadena, CA 91109, USA}
\blfootnote{E-mail addresses: \texttt{tbourdai@caltech.edu}, \texttt{pau@caltech.edu}, \texttt{yxjmath@caltech.edu}, \\ \texttt{rsb@caltech.edu}, \texttt{nicolas.f.rouquette@jpl.nasa.gov}, \texttt{owhadi@caltech.edu}}
\blfootnote{Classification: Physical Sciences, applied mathematics. Keywords:  Gaussian process,  analysis of variance, hypergraph discovery,  raw data analysis, functional relationships.  }

\maketitle

\begin{abstract}
Most problems within and beyond the scientific domain can be framed into one of the following three levels of complexity of function approximation. {\bf Type 1:} Approximate an unknown function given input/output data. {\bf Type 2:} Consider a collection of variables and functions, some of which are unknown, indexed by the nodes and hyperedges of a hypergraph (a generalized graph where edges can connect more than two vertices). Given partial observations of the variables of the hypergraph (satisfying the functional dependencies imposed by its structure), approximate all the unobserved variables and unknown functions. {\bf Type 3:} Expanding on Type 2, if the hypergraph structure itself is unknown, use partial observations of the variables of the hypergraph to discover its structure and approximate its unknown functions. 
These hypergraphs offer a natural platform for organizing, communicating, and processing computational knowledge. While most scientific problems can be framed as the data-driven discovery of unknown functions in a computational hypergraph whose structure is known (Type 2), many require the data-driven discovery of the structure (connectivity) of the hypergraph itself (Type 3).
We introduce an interpretable Gaussian Process (GP) framework for such (Type 3) problems that does not require randomization of the data, nor access to or control over its sampling, nor sparsity of the unknown functions in a known or learned basis. Its polynomial complexity, which
contrasts sharply with the super-exponential complexity of causal inference methods, is enabled by the  nonlinear analysis of variance capabilities of GPs used as a sensing mechanism.

\end{abstract}

\subsection*{Significance statement.} 
  "Civilization advances by extending the number of important operations we can perform without thinking about them" (Whitehead, 1911). In line with this perspective, many complex data analysis problems within and beyond the scientific domain involve discovering graphical structures and functional relationships within data. Nonlinear variance decomposition with Gaussian Processes simplifies and automates this process. Other methods, such as Artificial Neural Networks, lack this variance decomposition feature. Information-theoretic and causal inference methods suffer from super-exponential complexity with respect to the number of variables. The proposed technique performs this task in polynomial complexity. This unlocks the potential for applications involving the identification of a network of hidden relationships between variables without a parameterized model at an unprecedented scale, scope, and complexity.

\section*{Introduction}

\subsection*{The three levels of complexity of function approximation.}
As illustrated in Fig.~\ref{fignature1}.(a-c), Type 1, Type 2 and Type 3 problems can be formulated as completing or discovering hypergraphs where nodes represent variables and edges represent functional dependencies. The graph in Type 1 has only two variables and one unknown function. The graph in Type 2 has multiple variables and (some possibly unknown) functions, and the connectivity of the graph is known. The graph in Type 3 has an unknown connectivity (functional dependencies between variables may be unknown) and this is the focus of this work.
Current methods for solving Type 1 and 2 problems include  Deep Learning (DL) methods, which benefit from extensive hardware and software support but have limited guarantees.
Despite their prevalence,  Type 3 challenges have been largely overlooked due to their inherent complexity. 
Causal inference methods \cite{morgan2015counterfactuals, glymour2016causal} and probabilistic graphs \cite{stegle2010probabilistic, lopez2015towards}  and sparse regression methods \cite{doostan2011non, brunton2016discovering}, offer potential avenues for addressing Type 3 problems. However, it is important to note that their application to these problems necessitates additional assumptions.
Causal inference models, for instance, typically assume randomized data and some level of access to the data generation process or its underlying distributions. Sparse regression methods, on the other hand, rely on the assumption that functional dependencies have a sparse representation within a known basis.
In this paper, we do not impose these assumptions, and thus, these particular techniques may not be applicable. 
Furthermore while the complexity of Bayesian causal inference methods may grow super-exponentially with the number $d$ of variables, the complexity of our method is that of $d$ parallel computations of polynomial complexities  bounded  between $\mathcal{O}(d)$ (best case)  and $\mathcal{O}(d^4)$ (worst case).

\begin{figure}[h]
    \centering
        \includegraphics[width=1\textwidth ]{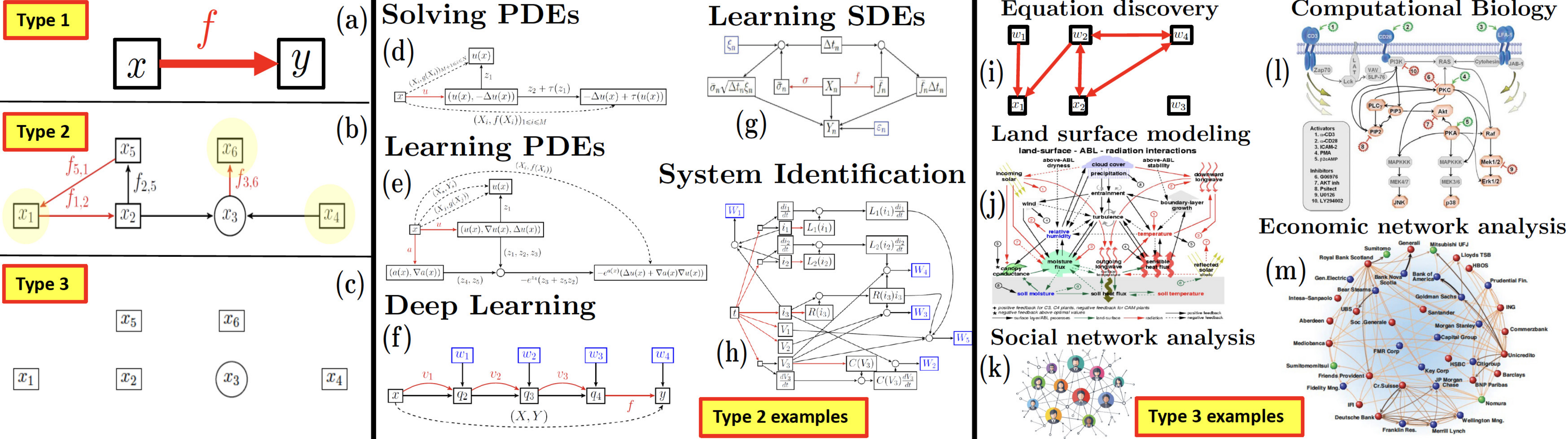}
    \caption{The three levels of complexity of function approximation.}
    \label{fignature1}
\end{figure}

\subsection*{Generalizing Gaussian Process methods.}
Although Gaussian Process (GP) methods are sometimes perceived as a well-founded but old technology limited to curve fitting (Type 1 problems), they have recently been generalized, beyond Type 1 problems, to an interpretable framework (Computational Graph Completion or CGC \cite{owhadi2022computational}) for solving Type 2 problems  \cite{owhadi2023ideas, chen2021solving,  batlle2023error, chen2023sparse, darcy2023one, hamzi2023learningpartial}, all while maintaining the simple and transparent theoretical and computational guarantees of kernel/optimal recovery methods  \cite{micchelli1977survey, OwhScobook2018}.
This paper introduces a comprehensive GP framework for solving Type 3 problems, which is interpretable and amenable to analysis. 
This framework leverages the Uncertainty Quantification (UQ) properties of GP methods, which do not have an immediate natural counterpart in DL methods.
It is based on a kernel generalization \cite{wahba2003introduction,owhadi2019kernelmd} of variance-based sensitivity analysis guiding the discovery of the structure of the hypergraph. Here, variables are linked via GPs, and those contributing to the highest data variance unveil the hypergraph's structure.  This GP variance decomposition of the data leads to signal-to-noise and a Z-score that can be employed to determine whether a given variable can be approximated as a nonlinear function of a 
subset of other variables.

\subsection*{The scope of Type 1, 2 and 3 problems.}
The scope of Type 1, 2 and 3 problems is immense. 
Numerical approximation \cite{OwhScobook2018, OwhScoSchNotAMS2019, SchaeferSullivanOwhadi17, schafer2021sparse}, Supervised Learning, and Operator Learning \cite{KFforcimate2020, hamzi2021learning, schafer2021sparsematrixvector, batlle2023kernel} can all be formulated as {\bf Type 1 problems}, i.e.,  as approximating unknown functions given  (possibly with noisy and infinite/high-dimensional) inputs/output data. 
The common GP based solution to these problems is to replace the underlying unknown function by a GP and compute its MAP estimator given available data.
{\bf Type 2 problems} include (Fig.~\ref{fignature1}.(d-h)) solving and learning (possibly stochastic) ordinary or partial differential equations  \cite{chen2021solving,darcy2023one}, Deep Learning  \cite{owhadi2023ideas}, dimension reduction, reduced-ordered modeling, system identification  \cite{owhadi2022computational}, closure modeling, etc. Indeed, all these problems can be formulated as completing a computational graph  \cite{owhadi2022computational}. In this formulation,  variables and functions are represented by the nodes and the edges of the graph whose structure corresponds to the functional dependencies between variables. Some of the functions and variables may be unknown, and by completing, we mean  approximating the unknown functions (colored in red in Fig.~\ref{fignature1}) given samples from the observed variables.
The  common GP-based solution to Type 2 problems is to simply replace unknown functions by  GPs and compute their MAP/MLE estimators given available data and constraints imposed by the structure of the graph  \cite{owhadi2022computational}.
While most problems in Computational Sciences and Engineering (CSE) and Scientific Machine Learning (SciML) can be framed as Type 1 and Type 2 challenges, many problems in science can only be categorized as {\bf Type 3 problems}, i.e., discovering the structure/connectivity of the graph itself from data prior to its completion. 
Indeed the scope of Type 3 problems extends well beyond Type 2 problems and includes equation discovery   (Fig.~\ref{fignature1}.(i));   the modeling of land surface interactions in weather prediction (Fig.~\ref{fignature1}.(j) from \cite{dirmeyer2019land}, discovering possibly hidden functional dependencies between state variables for a finite number of snapshots of those variables); social network analysis (Fig.~\ref{fignature1}.(k)  from \cite{gittell2021relational}, discovering functional dependencies between quantitative markers associated with each individual in situations where the connectivity of the network may be hidden); economic network analysis  (Fig.~\ref{fignature1}.(m)  from \cite{schweitzer2009economic},  discovering functional dependencies between the economic markers of different agents or companies, which is significant to systemic risk analysis); and computational biology (Fig.~\ref{fignature1}.(l) from \cite{sachs2005causal}, identifying pathways and interactions between genes from their expression levels).

\begin{figure}[h]
    \centering
        \includegraphics[width=1\textwidth ]{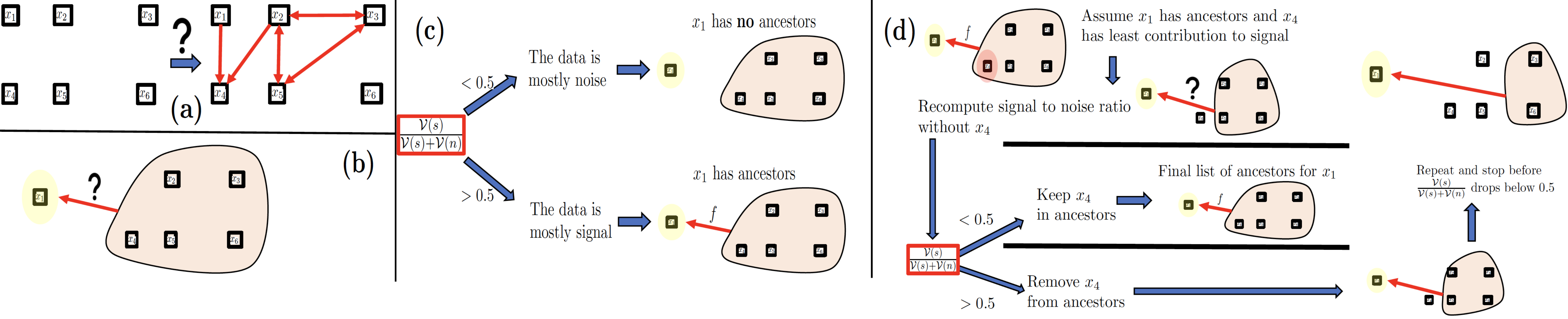}
    \caption{Ancestors identification in Type 3 problem.}
    \label{fignature2}
\end{figure}

\section*{Overview of the proposed approach for Type 3 problems.}

We first present an algorithmic overview of the proposed GP-based approach for Type 3 problems. For ease of presentation, we consider the 
simple setting of Fig.~\ref{fignature2}.(a) where we are given $N$ samples on the variables $x_1,\ldots,x_6$. After measurements/collection, these variables are normalized to have zero mean and unit variance. Our objective is to uncover the underlying dependencies between them.

\subsection*{A signal-to-noise ratio to decide whether or not a node has ancestors.}
Our algorithm's core concept is the identification of ancestors for each node in the graph. Let's explore this idea in the context of a specific node, say $x_1$, as depicted in Fig.~\ref{fignature2}(b). Determining whether $x_1$ has ancestors is akin to asking if $x_1$ can be expressed as a function of $x_2, x_3, \ldots, x_6$. In other words, can we find a function $f$ (living in a pre-specified space of functions that could be of controlled regularity) such that:
\begin{equation}\label{eqiheuedguydd}
x_1\approx f(x_2,\ldots,x_6)\, ?
\end{equation}  
To answer this question we  regress $f$ with a centered  GP $\xi\sim \cN(0,\Gamma)$ whose covariance function $\Gamma$ is an additive kernel of the form
$
\Gamma=K_s+\gamma \updelta(x-y)\,,
$
where $K_s$ is a smoothing kernel, $\gamma>0$ is a regularization parameter and $\updelta(x-y)$ is the white noise covariance  operator. This is equivalent to assuming the GP $\xi$ to be the sum of two independent GPs, i.e., $\xi=\xi_s+\xi_n$ where $\xi_s\sim \cN(0,K_s)$  is a smoothing/signal  GP and $\xi_n\sim \cN\big(0,\gamma \updelta(x-y)\big)$ is a noise GP.
Writing $\H_{K_s}$ for the Reproducing Kernel Hilbert Space (RKHS) induced by the kernel $K_s$, this is also equivalent to approximating $f$ with a minimizer of 
\begin{equation}\label{eqresidual}
\inf_{f\in \H_{K_s}}\|f\|_{K_s}^2+\frac{1}{\gamma}\big\|f(X)-Y\big\|_{\R^N}^2\,,
\end{equation}
where $\|\cdot\|_{\R^N}^2$ is the Euclidean norm on $\R^N$, $X$ is the input data on $f$ obtained as an $N\times 5$-matrix whose rows $X_i$ are the samples on $x_2,\ldots,x_6$, $Y$ is the output data on $f$ obtained as an $N$-vector whose entries are obtained from the samples on $x_1$, and $f(X)$ is a $N$-vector whose entries are the evaluations $f(X_i)$.
At the minimum
\begin{equation}
\V(s):=\|f\|_{K_s}^2
\end{equation}
quantifies the data variance explained by the signal GP $\xi_s$ and 
\begin{equation}
\V(n):=\frac{1}{\gamma}\big\|f(X)-Y\big\|_{\R^N}^2
\end{equation}
quantifies the data variance explained by the noise GP $\xi_n$ \cite{owhadi2019kernelmd}. 
This allows us to define the signal-to-noise ratio
\begin{equation}
\frac{\V(s)}{\V(s)+\V(n)}\in [0,1]\,.
\end{equation}
If  $\frac{\V(s)}{\V(s)+\V(n)}<0.5$\footnote{We will later present a version with a more sophisticated method for pruning, but we keep the 0.5 threshold in this example for simplicity.}, then, as illustrated in Fig.~\ref{fignature2}.(c), we deduce that $x_1$ has no ancestors, i.e., $x_1$ cannot be approximated as function of $x_2,\ldots,x_6$. Conversely if   $\frac{\V(s)}{\V(s)+\V(n)}>0.5$, then, we deduce that $x_1$ has  ancestors, i.e., $x_1$ can be approximated as function of $x_2,\ldots,x_6$.

\subsection*{Selecting the signal kernel \texorpdfstring{$K_s$}.}
 This process is repeated by selecting the kernel  $K_s$ to be linear ($K_s(x,x')=1+\beta_1 \sum_i x_i x_i' $), quadratic ($K_s(x,x')=1+\beta_1 \sum_i x_i x_i' + \beta_2 \sum_{i\leq j} x_i x_j x_i' x_j' $) or fully nonlinear to identify $f$ as linear, quadratic, or nonlinear. In the case of a nonlinear kernel, we employ:
\begin{equation}\label{eqrvhgvjhgv}
K_s(x,x')=1+\beta_1 \sum_i x_i x_i' + \beta_2 \sum_{i\leq j} x_i x_j x_i' x_j' + \beta_3 \prod_i (1+k(x_i,x_i'))
\end{equation}
where $k$ is a universal kernel, such as a Gaussian or a Mat\'{e}rn kernel, with all parameters set to $1$, and $\beta_i$ assigned the default value $0.1$.
We select $K_s$ as the first kernel that surpasses a signal-to-noise ratio of 0.5. If no kernel reaches this threshold, we conclude that $x_1$ lacks ancestors.

\subsection*{Pruning ancestors based on signal-to-noise ratio.}
Once we establish that $x_1$ has ancestors, the next step is to prune its set of ancestors iteratively. We remove nodes with the least contribution to the signal-to-noise ratio and stop before that ratio drops below $0.5$ as illustrated in Fig.~\ref{fignature2}.(d). To describe this, assume that $K_s$ is as in \eqref{eqrvhgvjhgv}. 
Then $K_s$ is an additive kernel that can be decomposed into two parts:
\begin{equation}
K_s=K_1+K_2\,,
\end{equation}
where $K_1=1+\beta_1 \sum_{i\not=1,2} x_i x_i' + \beta_2 \sum_{i\leq j, i,j\not=1,2} x_i x_j x_i' x_j' + \beta_3 \prod_{i\not=1,2} (1+k(x_i,x_i'))$ does not depend on $x_2$ and
$K_2=K_s-K_1$ depends on $x_2$.
This decomposition allows us to express $f$ as the sum of two components:
\begin{equation}\label{eqmodecstep}
f=f_1+f_2\,,
\end{equation}
where $f_1$ does not depend on $x_2$, $f_2$ depends on $x_2$ and
$
(f_1,f_2)= \operatorname{argmin}_{(g_1,g_2)\in \H_{K_1}\times \H_{K_2}\text{ s.t. } g_1+g_2=f}\|g_1\|_{K_1}^2+\|g_2\|_{K_2}^2\,.
$
Furthermore,
$
\|f\|_{K_s}^2=\|f_1\|_{K_1}^2+\|f_2\|_{K_2}^2\,,
$
and $\frac{\|f_2\|_{K_1}^2}{\|f\|_{K_s}^2}\in [0,1]$ quantifies the contribution of $x_2$ to the signal data variance.
Following the procedure illustrated in Fig.~\ref{fignature2}.(d), if, for example, $x_4$ is found to have the least contribution to the signal data variance, we recompute the signal-to-noise ratio without $x_4$ in the set of ancestors for $x_1$. If that ratio is below $0.5$, we do not remove $x_4$ from the list of ancestors, and $x_2,x_3,x_4,x_5, x_6$ is the final set of ancestors of $x_1$. 
If this ratio remains above $0.5$, we proceed with the removal. This iterative process continues, and we stop before the signal-to-noise ratio drops below $0.5$ to identify the final list of ancestors of $x_1$.
The most efficient version of our proposed algorithm does not use a threshold of $0.5$ on the signal-to-noise ratio to prune ancestors, but it rather employs an inflection point in the 
noise-to-signal ratio  $\frac{\V(n)}{\V(s)+\V(n)}(q)$ as a function of the number $q$ of ancestors (Fig.~\ref{fignature3}.(d)). To put it simply, after ordering the ancestors in decreasing contribution to the signal, the final number $q$ of ancestors is determined as the maximizer of $\frac{\V(n)}{\V(s)+\V(n)}(q+1)-\frac{\V(n)}{\V(s)+\V(n)}(q)$. 

\subsection*{Computational complexity} 
We will now present a detailed analysis of the computational demands of the proposed method as a function of the number of variables, denoted as $d$, and the number of samples, $N$, pertaining to these variables.
In the worst case, the proposed approach necessitates, for each of the $d$ variables: 
for $i=1,\ldots,d-1$, regressing a function mapping $d-i$ variables to the variable of interest and performing a mode decomposition, as exemplified in \eqref{eqmodecstep}, to identify the variable with the minimal contribution to the signal. 
Since these  two steps have the same cost, it follows that, in the worst case, the total computational complexity of the proposed method is $\mathbf{\mathcal{O}(d^2N^3)}$ which corresponds to product of the number of double-looping operations, $d^2$, and the cost of kernel regression from $N$ samples which, without acceleration, is $N^3$ (i.e., the cost of  inverting a $N\times N$ dense kernel matrix).
However, if kernel scalability techniques are utilized, such as when the kernel has a rank $k$ (for example, $k=d$ if the kernel is linear) or is approximated by a kernel of rank $k$ (e.g., via a random feature map), then this worst-case bound can be reduced to $\mathbf{\mathcal{O}(d^2 N k^2)}$ by reducing the complexity of each regression step from $\mathcal{O}(N^3)$ to $\mathcal{O}(N k^2)$.
Note that the statistical accuracy of the proposed approach requires that $N > d$ if the dependence of the unknown functions on their inputs is not sparse. 
Moreover, in the absence of kernel scalability techniques, the worst-case memory footprint of the method is $\mathcal{O}(N^2)$ due to the necessity of handling dense kernel matrices. However, once the functional ancestors of each variable are determined, these matrices can be discarded. Consequently, only one such matrix needs to be retained in memory at any given time.

\begin{figure}[h]
    \centering
        \includegraphics[width=\textwidth ]{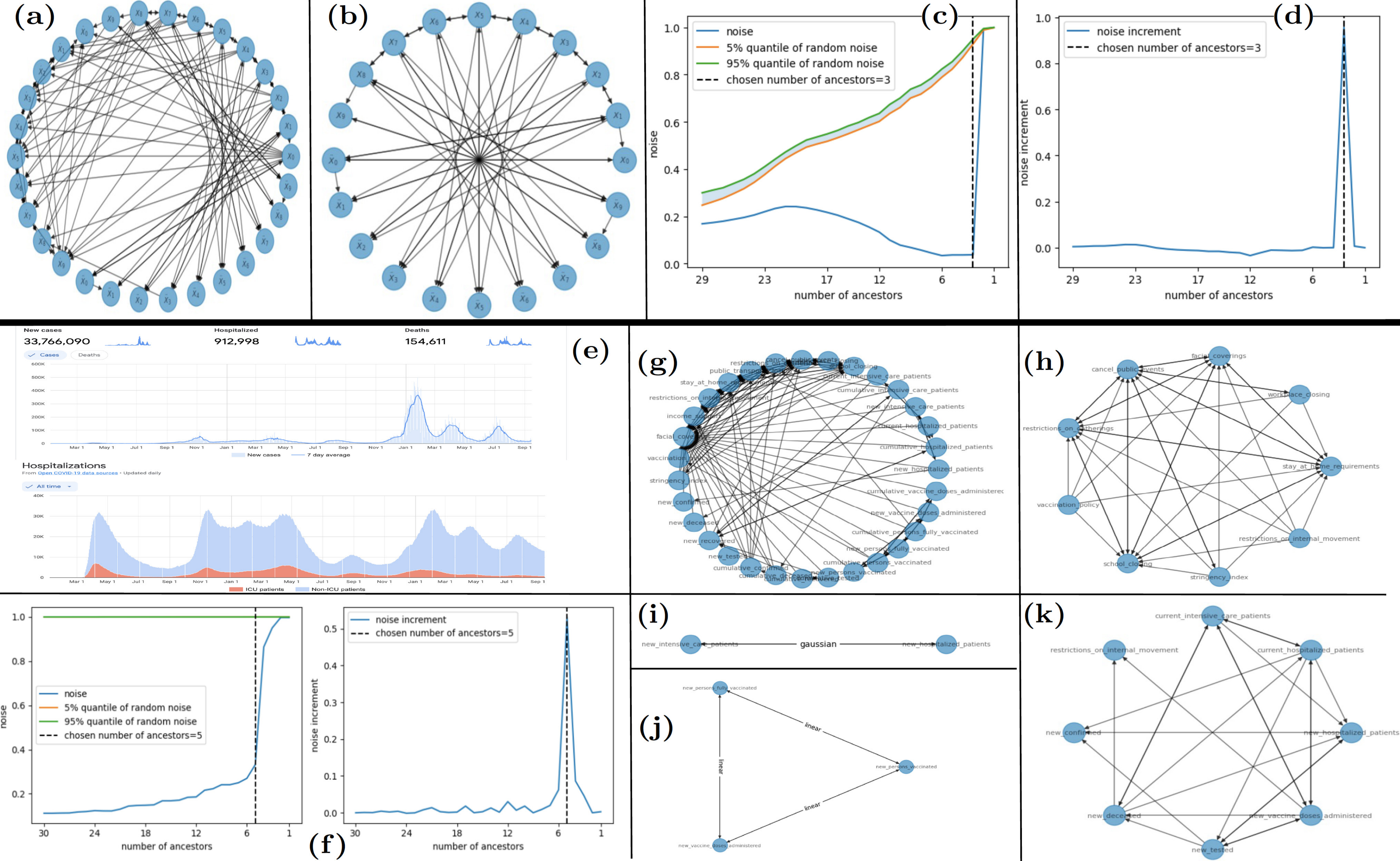}
    \caption{(a-d)  The Fermi-Pasta-Ulam-Tsingou system. (e-k) The Google Covid 19  open data.}
    \label{fignature3}
\end{figure}

\section*{Results}

The following  examples and experiments illustrate the proposed approach.

\subsection*{The Fermi-Pasta-Ulam-Tsingou system}
The Fermi-Pasta-Ulam-Tsingou (FPUT) system \cite{palais1997symmetries} is a prototypical chaotic dynamical system. 
 It is composed of $M$ masses indexed by $j\in\{0,\ldots,M-1\}$ with equilibrium position $jh$ with $h=1/M$. Each mass is tethered  to its two adjacent masses  by a nonlinear spring, and the displacement of the mass $x_j$ adheres to the equation: 
\begin{equation}
    \ddot{x}_j=\frac{c^2}{h^2}(x_{j+1}+x_{j-1}-2x_j) \left (1+\alpha(x_{j+1}-x_{j-1}) \right )\,,
\end{equation}
where $\alpha(x) = x^2$, $c=1$ and $M=10$. We use fixed boundary conditions by adding two more masses, with $x_{-1}=x_{M}=0$. 
We take a total of $1000$ snapshots from multiple trajectories  and the observed variables are the positions, velocities, and accelerations of all the underlying masses.
 In the graph discovery phase, every other node is initially deemed a potential ancestor for a specified node of interest. We then proceed to iteratively remove the node with the least signal contribution. 
 The step resulting in the largest surge in the noise-to-signal ratio is inferred as one eliminating a crucial ancestor, thereby pinpointing the final ancestor set.
 Fig.~\ref{fignature3}.(c) shows a plot of the noise-to-signal ratio $\frac{\V(n)}{\V(s)+\V(n)}(q)$ as a function of the number $q$ of proposed ancestors for the variable $\ddot{x}_7$ and with $Z$-test quantiles (in the absence of signal, the noise-to-signal ratio should fall within the shaded area with probability $0.9$).  Removing a node essential to the equation of interest causes the noise-to-signal ratio to markedly jump from approximately 25\% to 99\%. 
  Fig.~\ref{fignature3}.(d) shows a plot of the noise-to-signal ratio increments $\frac{\V(n)}{\V(s)+\V(n)}(q)-\frac{\V(n)}{\V(s)+\V(n)}(q-1)$ as a function of the number $q$ of ancestors for the variable $\ddot{x}_7$. Note that
  the increase in the noise-to-signal ratio is significantly higher compared to previous removals when an essential node was removed. 
Therefore, while solely relying on a fixed threshold to decide when to cease the removals might prove challenging,  evaluating the increments in noise-to-signal ratios offers a clear guideline for efficiently and reliably pruning ancestors.
The recovered full graph, depicted in Fig.~\ref{fignature3}.(a),  is  remarkably accurate despite the nonlinear nature of the model and the fact that our prior only encodes that the nonlinearity is smooth. Therefore, our algorithm does not require a dictionary or extensive knowledge of the structure of the unknown functions.
Notably, velocity variables are accurately identified as non-essential and omitted from the ancestors of position and acceleration variables. 
Fig.~\ref{fignature3}.(b), which omits velocity variables for clarity, further elucidates the accurate recovery of dependencies.
The dependencies are the simplest and clearest possible. They match exactly those of the original equations except for the boundary particles for which we recover valid equivalent equations.

\subsection*{The Google Covid 19  open data.}
Consider the COVID-19 data from Google\footnote{The dataset can be accessed \href{https://health.google.com/covid-19/open-data/raw-data}{here}}. We focus on a single country, France, to ensure consistency in the data and avoid considering cross-border variations that are not directly reflected in the data. We select 31 variables that describe the state of the country during the pandemic, spanning over 500 data points, with each data point corresponding to a single day. 
These variables are  categorized as the following datasets: (1) Epidemiology dataset: Includes quantities such as new infections, cumulative deaths, etc.
(2) Hospital dataset: Provides information on the number of admitted patients, patients in intensive care, etc.
(3) Vaccine dataset: Indicates the number of vaccinated individuals, etc.
(4) Policy dataset: Consists of indicators related to government responses, such as school closures or lockdown measures, etc.
Some of these variables are illustrated in Fig.~\ref{fignature3}.(e). The problem is then to analyze this data and identify possible hidden functional relations between these variables. 
Fig.~\ref{fignature3}.(f) shows  the noise-to-signal ratio  (and its increments) as function of the number of ancestors of the ``cumulative number of hospitalized patients'' variable.
Even for this real dataset, the proposed approach gives a clear signal for stopping the pruning process.
 Fig.~\ref{fignature3}.(g) shows the full recovered graph, which is highly clustered.  Fig.~\ref{fignature3}.(h) shows the cluster corresponding to the variable ``schools closing'' revealing 
 that the government either implemented multiple restrictive measures simultaneously or lifted them in unison (except for mask mandates that were on the verge of being identified as noise).
The vaccination cluster (Fig.~\ref{fignature3}.(j)) reveals a linear relationship between variables (signaling redundant information) and the hospitalization cluster  (Fig.~\ref{fignature3}.(i)) reveals a nonlinear one.
 Eliminating redundant nodes leads to the sparse graph shown in Fig.~\ref{fignature3}.(k), which is interpretable and amenable to (both quantitative and qualitative) analysis,

\begin{figure}[h]
    \centering
        \includegraphics[width=\textwidth ]{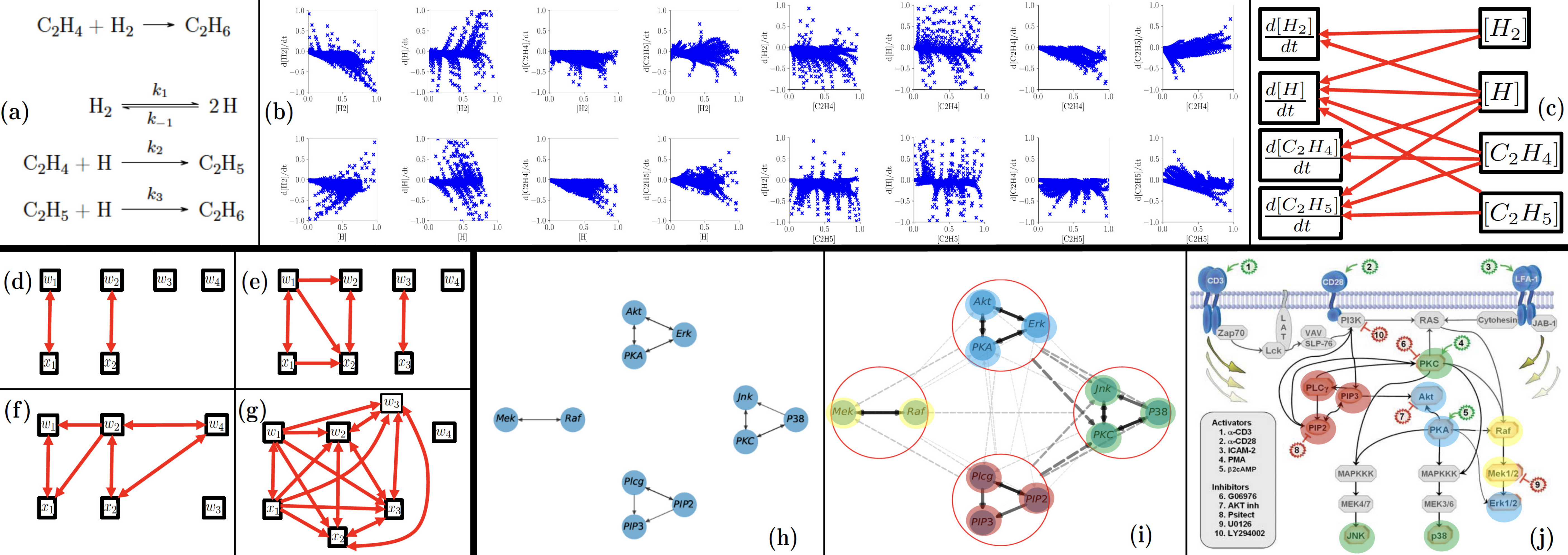}
    \caption{(a-c) Chemical reaction network. (d-g)  Algebraic equations. (h-j) Cell signaling network.}
    \label{fignature4}
\end{figure}

\subsection*{Chemical reaction network.}
In this example, we consider the recovery of a chemical reaction network from concentration snapshots. The reaction network, illustrated in Fig.~\ref{fignature4}.(a) is that of
 the hydrogenation of ethylene $(\text{C}_2 \text{H}_4)$ into ethane $(\text{C}_2 \text{H}_6)$.
The problem is that of recovering the underlying chemical reaction network from snapshots (illustrated in Fig.~\ref{fignature4}.(b)) of concentrations  $[H_2]$, $[H]$, $[C_2H_4]$ and $[C_2H_5]$ and their time derivatives.
 $\frac{d[H_2]}{dt}$,
$\frac{d[H]}{dt}$, $\frac{d[C_2H_4]}{dt}$ and $\frac{d[C_2H_5]}{dt}$.   The proposed approach leads to a perfect recovery of the computational graph (shown in Fig.~\ref{fignature4}.(c)) and 
  a correct identification of quadratic functional dependencies between variables.

\subsection*{Algebraic equations.}
Fig.~\ref{fignature4}.(a-d) illustrate the application of the proposed approach to the recovery of functional dependencies from data satisfying hidden algebraic equations. In all these examples, we have  $d=6$ or $d=7$ variables and $N=1000$ samples from those variables. 
For $d=6$ the variables are $w_1, w_2, w_3, w_4, x_1, x_2$. For $d=7$ the variables are $w_1, w_2, w_3, w_4, x_1, x_2, x_3$.
The samples from the variables $w_1$ to $w_4$ are i.i.d. $\N(0,1)$ random variables, and the samples from $x_1$, $x_2$ (and $x_3$ for $d=7$) are functionally dependent on the other variables.
In the first example, $d=6$ and the samples from $x_1$ and $x_2$ satisfy the equations $x_1=w_1$ and $x_2=w_2$. The algorithm selects the linear kernel and 
Fig.~\ref{fignature4}.(a) shows the  recovered graph (which is exact).
In the second example, $d=7$ and the samples from $x_1, x_2$ and $x_3$ satisfy the equations $x_1=w_1$,  $x_2=x_1^2+1+0.1w_2$, and $x_3=w_3$. The algorithm selects the quadratic kernel and 
Fig.~\ref{fignature4}.(b) shows the  recovered graph (which is exact). Even though $x_2$ can trace back its origin to either $x_1$ and $w_2$ or $w_1$ and $w_2$, the algorithm recognizes $x_1$, $w_1$, and $w_2$ as its ancestors underscoring the importance of eliminating redundant variables when aiming at deriving the sparsest graph. 
In the third example, $d=6$ and the samples from $x_1$ and $x_2$  satisfy the equations $x_1=w_1w_2 $ and $ x_2=w_2\sin(w_4)$. The algorithm selects the nonlinear kernel and 
Fig.~\ref{fignature4}.(c) shows the  recovered graph (which is exact).
In the fourth example,  $d=7$ and the samples from $x_1, x_2$ and $x_3$ satisfy the equations $x_1=w_1$, $x_2=x_1^3+1+0.1w_2 $ and $x_3=(x_1+2)^3+0.1w_3$. Although these equations appear to be cubic, the 
algorithm correctly selects the quadratic kernel and makes an exact recovery of the graph shown in Fig.~\ref{fignature4}. (d) revealing hidden quadratic dependencies between variables.

\subsection*{Cell signaling network}\label{seccellsignaling}
Next, we apply the proposed framework to the example illustrated in Fig.~\ref{fignature1}.(l) from \cite{sachs2005causal} and 
discover a hierarchy of functional dependencies in biological cellular signaling networks. We use single-cell data consisting of the $d=11$ phosphoproteins and phospholipids levels in the human immune system T-cells that were measured using flow cytometry. This dataset was studied from a probabilistic modeling perspective in previous works. While \cite{sachs2005causal} learned a directed acyclic graph to encode causal dependencies, \cite{friedman2008sparse} learned an undirected graph of conditional independencies between the $d$ molecule levels by assuming the underlying data follows a multivariate Gaussian distribution. The latter analysis encodes acyclic dependencies but does not identify directions. In this work, we aim to identify the functional dependencies without imposing strong distributional assumptions on the data. We simply use $N=2,000$ samples chosen uniformly at random from the dataset consisting of $11$ proteins and $7446$ samples of their expressions.
We apply the algorithm in two stages. The first stage of the algorithm uses  only linear and quadratic kernels and recovers the graph shown in Fig.~\ref{fignature4}.(h). It consists 
 of four disconnected clusters where the molecule levels in each cluster are closely related by linear or quadratic dependencies (all connections are linear except for the connection between Akt and PKA, which is quadratic). These edges match a subset of the edges found in the gold standard model identified  in~\cite{sachs2005causal}. With perfect noiseless dependencies, one can define constraints that reduce the total number of variables in the system.
 Second, we learn the connections between groups of variables within each cluster with nonlinear kernels and obtain the graph shown in Fig.~\ref{fignature4}.(i) in which 
 solid arrows indicate strong intra-cluster connections identified in the first level, and dashed lines indicate weaker connections between nodes and clusters identified in the second level. 
The width and grayscale intensities of each edge correspond to its signal-to-noise ratio.
We emphasize that while causal graph recovery methods rely on the control of the sampling of the underlying variables (i.e., the simultaneous measurement of multiple phosphorylated protein and phospholipid components in thousands of individual primary human immune system cells, and perturbing these cells with molecular interventions), the reconstruction obtained by our method did not use this information and recovered functional dependencies rather than causal dependencies. 
Interestingly, the information recovered through our method appears to complement and enhance the findings presented in~\cite{sachs2005causal} (e.g., the linear and noiseless dependencies between variables in the JNK cluster is not something that could easily be inferred from the graph produced in~\cite{sachs2005causal} shown in Fig.~\ref{fignature1}.(j) where we have colored the clusters for comparison).\\
\noindent{{\bf Comparisons}. Using the expected graph reported in~\cite{sachs2005causal} as the ground truth (acknowledging that it may not be entirely accurate), we compare the edges our approach incrementally added to the true graph. Figure~\ref{figbiochem}.(a) reports the number of additional edges that have been added and are not present in the ground truth (false positives) and edges removed that are present in the ground truth graph (false negatives). The added edges are based on the two-stage procedure described above, where we first add the ten intra-cluster connections, followed by inter-cluster connections. Edges are added in decreasing order of signal-to-noise ratio, starting with the strongest. In the reported results, we do not account for the recovery of the direction of ground-truth edges. We note that, up to direction, all intra-cluster connections, along with the inter-cluster connections with the strongest signals are found in the ground truth graph, leading to the initial decrease in false negatives with only one false positive edge (the linear connection P38 $\rightarrow$ Jnk that is not reported in the true graph). With the addition of the remaining (possibly non-spurious) edges, the number of false negatives drops to one, having recovered all edges, except for the one between PKC and Raf, which is identified to be statistically non-informative in our approach.}

\begin{figure}[h]
    \centering
        \includegraphics[width=\textwidth ]{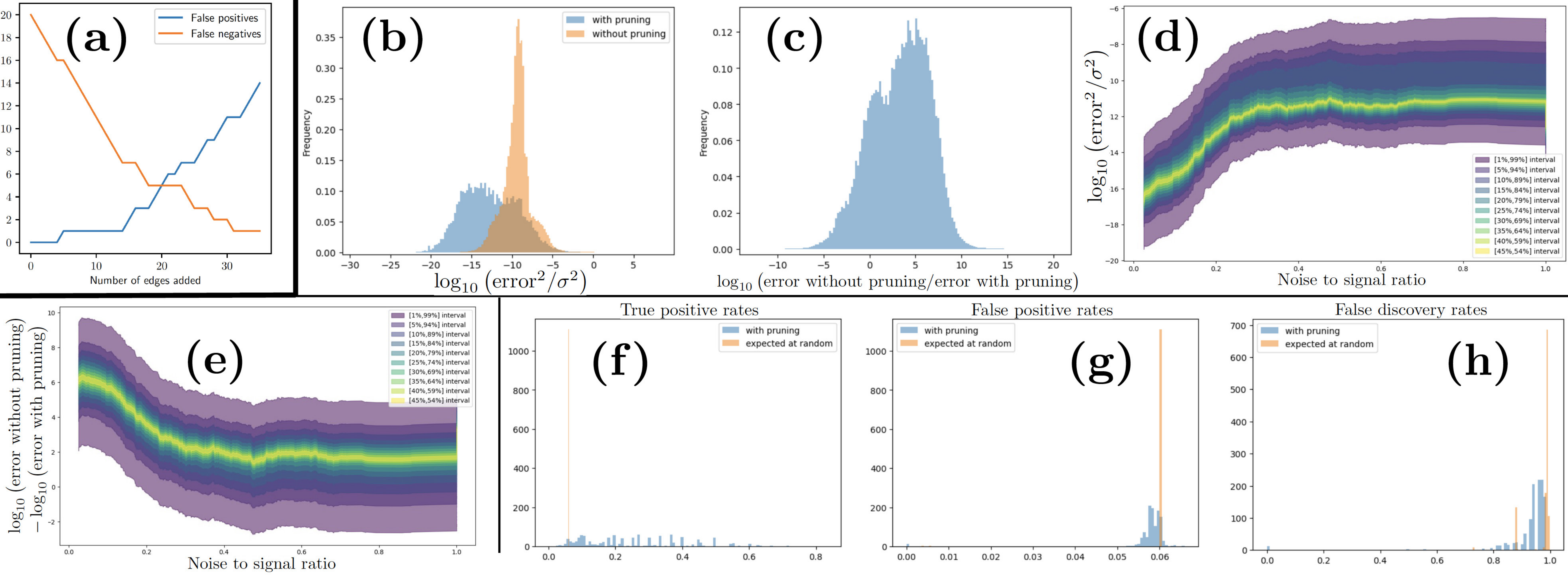}
    \caption{(a) Cell signaling network comparisons.
    (b-h) The BCR reaction benchmark.}
    \label{figbiochem}
\end{figure}

\subsection*{A large-scale chemical reaction network: the BCR reaction benchmark}
Lastly, we stress-test the scalability of our approach by applying it to a large-scale chemical reaction network: the BCR reaction benchmark from \cite{Loman_Ma_Ilin_Gowda_Korsbo_Yewale_Rackauckas_Isaacson_2023}, which encompasses 1122 species. The dataset comprises 2400 snapshots of species concentrations and their corresponding time derivatives. We leveraged JAX's inherent parallelization capabilities \cite{jax2018github} to accelerate our computations, allowing for the simultaneous pruning of multiple nodes while abstracting the complexity of parallel execution.
While the scaling with respect to the number of data points is straightforward, scaling with the number of variables  introduces a trade-off between computational speed and memory footprint. Specifically, the process of identifying the ancestors of various nodes can be expedited by storing a large array for all nodes.
Using a DGX workstation equipped with four Nvidia V100 GPUs, each with 32GB of memory, pruning 190 nodes took approximately three days, projecting a total experiment duration of around one month. Nonetheless, we can mitigate this computational burden by optimizing the  computation of terms of the form $y^T K y$ for the specific quadratic kernel identified for this example. We include the details of such optimization in the supplementary material. By implementing this optimization, the duration of the entire experiment was reduced to just one hour.

In the first experiment, we simulated five trajectories of the associated system of ODEs, recording 1000 snapshots per trajectory. Out of these 5000 snapshots, 2400 were randomly selected as training data, and 2600 as testing data. 
Writing TP, TN, FP and FN for True/False Positives/Negatives and using the metrics True Positive Rate (TPR$=$TP$/$(TP$+$FN)), False Positive Rate (FPR$=$FP$/$(FP$+$TN)), and False Discovery Rate (FDR$=$FP$/$(TP$+$FP)), we observed a TPR of 39.9\%, an FPR of 16.4\%, and an FDR of 97.2\% (indicating that 97.2\% of predicted positives are false). 
This high FDR can be attributed to the limited exploration of the full variable range—1122 in total—by the five trajectories. The trajectories explored a subset of the possible space (near a limit cycle attractor), which led to the recovery of functional dependencies that represent both the chemical reactions and the specific subspace visited.
Furthermore, with 1122 variables, the 630,003 coefficients of the underlying quadratic equations are vastly under-determined with only 2400 data points.
Despite the high FDR in the recovered graph,  as illustrated in Fig.~\ref{figbiochem}.(b-c), the CHD pruning process vastly improves the accuracy (by orders of magnitude) of the estimated functions on the 2600 unseen snapshots by reducing the dimension of the regression problem whenever possible. 
We denote $y_i$ as an observed data point, $\sigma^2$ as the variance of the observed data, $\hat{y}_i$ for a predicted data point without pruning, and $\bar{y}_i$ for a predicted data point post-pruning. Fig.~\ref{figbiochem}.(b) illustrates the histogram of the log-normalized squared errors before and after pruning, expressed as $\log_{10}\big(|y_i-\hat{y}_i|^2/\sigma^2\big)$ and $\log_{10}\big(|y_i-\bar{y}_i|^2/\sigma^2\big)$. The 99th percentile of the normalized squared error is less than $10^{-2}$ for all species.
Fig.~\ref{figbiochem}.(c) displays the histogram of the log-normalized squared error improvements due to pruning, calculated as $\log_{10}\big(|y_i-\hat{y}_i|^2/|y_i-\bar{y}_i|^2\big)$.
Fig.~\ref{figbiochem}.(d-e)
display the quantiles of the histograms post-pruning, conditioned on the noise-to-signal ratio observed at the final pruning step. These plots reveal a clear trend: a higher noise-to-signal ratio at the time of pruning correlates with increased error and diminished improvements in accuracy.

In a second experiment, we formed the data by randomly sampling concentrations uniformly in $[0,1]$  (independently across species and snapshots) and recorded the resulting time derivatives. While this sampling increased the variability of the 2400 snapshots, the model remained vastly underdetermined. The noise-to-signal and bootstrapped (Z-test) ratios  remained close to 0.5, 
suggesting insufficient data
for statistically significant variable importance assessments.
Nonetheless, as depicted in Fig.~\ref{figbiochem}.(f-h), significant insights can still be gleaned from the activations, showing notable improvements when comparing the histograms of the values of TPR, FPR, and FDR obtained with pruning based on these ratios and pruning at random.
 This analysis reveals that even with high dimensionality and scarce data, between 10\% and 80\% of the true ancestors can still be accurately identified.

\section*{Discussions}

\subsection*{Limitations}  In its present form, the proposed approach is limited by several factors. (1) Without access to the sampling of the data, the direction of some edges may not be identifiable. For instance the functional relationship $x-2y=0$ can be represented as both $y=2x$ ($x\rightarrow y$) and $x=y/2$ ($y\rightarrow x$). 
(2) It assumes an additive noise $W$ on the functional relationship $y=f(x)+W$ between the variables $x$ and $y$. 
In a fully probabilistic setting, this structure may be non-additive, i.e., of the form $y=f(x,W)$, which implies discovering a general transition kernel, i.e., a non-Gaussian generative model. Although our method achieves polynomial complexity, in settings where one has access to the distribution of the data, the price to pay, when compared with information-theoretic methods,  is a reduction in generality imposed by the stronger assumption made on the data-generating process. Furthermore, the price to pay for the weaker data requirements (i.e., the absence of interventional data) is that our method recovers functional relationships rather than 
causal ones or conditional dependencies. 
(3) If the (noisy) functional relationship $y=f(x)+W$ is associated with a non-regular (e.g., discontinuous) function $f$ then the kernels discussed above (linear, quadratic and fully nonlinear) will be misspecified and may lead to false negatives. 
The kernel selection and hyperparameter tuning problems in misspecified settings require further work. (4) 
As demonstrated in the BCR reaction application, while the method scales well computationally with an increase in the number of variables, it may still be impacted by the curse of dimensionality. This occurs particularly if the dataset only covers a limited subset of the full range of variable values.
 Given the results displayed in 
 Fig.~\ref{figbiochem}.(b-h) we suspect that this impact could be mitigated by adopting more advanced strategies in place of our current top-down pruning method. Such strategies could involve grouping variables and integrating both top-down and bottom-up iterative approaches.

\subsection*{Conclusions} 
We have developed a comprehensive Gaussian Process framework for solving Type 3 (hypergraph discovery) problems, which is interpretable and amenable to analysis. The breadth and complexity of Type 3 problems significantly surpass those encountered in Type 2 (hypergraph completion), and the initial numerical examples we present serve as a motivation for the scope of Type 3 problems and the broader applications made possible by this approach.
Our proposed algorithm is designed to be fully autonomous, yet it offers the flexibility for manual adjustments to refine the graph's structure recovery.
We emphasize that our proposed approach is not intended to supplant causal inference methods~\cite{pearl2009causality}; see Methods for a complete overview. Instead, it aims to incorporate a distinct kind of information into the graph's structure, namely, the functional dependencies among variables rather than their causal relationships. 
 Additionally, our method eliminates the need for a predetermined ordering of variables, a common requirement in acyclic probabilistic models where determining an optimal order is an NP-hard problem usually tackled using heuristic approaches. Furthermore, our approach can actually be utilized to generate such an ordering by quantifying the strength of the connections it recovers.
The Uncertainty Quantification properties of the underlying Gaussian Processes are integral to the method and could also be employed to quantify uncertainties in the structure of the recovered graph.
We also observe that forming clusters from highly interdependent variables helps to obtain a sparser graph. Additionally, the precision of the pruning process is enhanced by avoiding the division of node activation within the cluster among its separate constituents. We employed this strategy in the recovery of the gene expression graph in Fig.~\ref{fignature4}.(i).
Given the polynomial complexity of our method, promising avenues for future work include applications to large datasets in  genomics and in systems biology, particularly in the reconstruction and intervention of metabolic pathways. These applications benefit from the ability to handle large-scale datasets efficiently, enabling the analysis of complex biological networks.


\subsection*{Data availability}
The data in the paper and the Supplementary Information are available in the \href{https://github.com/TheoBourdais/ComputationalHypergraphDiscovery}{Github repository of the paper}.

\subsection*{Code availability}
The code for the algorithm and its application to various examples are available for download (and as as an installable python library/package) in the \href{https://github.com/TheoBourdais/ComputationalHypergraphDiscovery}{Github repository of the paper}.

\subsection*{Online content}
Supplementary Information is available for this paper.

{\small
\subsection*{Acknowledgements}
HO, TB, PB, XY, and RB acknowledge support from the Air Force Office of Scientific Research under MURI award number FA9550-20-1-0358 (Machine Learning and Physics-Based Modeling and Simulation). Additionally, HO, TB, and PB acknowledge support by the Department of Energy under award number DE-SC0023163 (SEA-CROGS: Scalable, Efficient and Accelerated Causal Reasoning Operators, Graphs and Spikes for Earth and Embedded Systems) and by the Jet Propulsion Laboratory, California Institute of Technology, under a contract with the National Aeronautics and Space Administration. HO and PB further acknowledge support by Beyond Limits (Learning Optimal Models) through CAST (The Caltech Center for Autonomous Systems and Technologies). 
TB acknowledges support from the Kortschak Scholar Fellowship Program. NR acknowledges support from the JPL Researchers on Campus (JROC) program. HO is grateful for the Department of Defense Vannevar Bush Faculty Fellowship. We are grateful to two referees for
 their insightful comments and valuable suggestions.}

\bibliographystyle{plain}
\bibliography{RPS}

\newpage

\section*{Supplementary information}

This supplementary document  provides an overview of refinements and generalizations on our proposed approach (Sec.~\ref{secadddetails}) detailed in subsequent sections.
It includes a summary of the principal components of our algorithm  (Sec.~\ref{secalgovertype3}). It includes 
a reminder on  Type 2 problems (Sec.~\ref{sectype2formal}) and their common GP-based solutions. It discusses the hardness of Type 3 problems, presents an 
  overview of causal inference methods, and a well-posed formulation of Type 3 problems (Sec.~\ref{sechardness}).
 Additionally, this document offers an in-depth description of our developed GP-based solution specifically designed for Type 3 problems (Section~\ref{bigsec4}), along with the corresponding algorithmic pseudo-codes (Section~\ref{Secpseudocode}). It also includes an analysis of the signal-to-noise ratio (SNR) test that is integral to our method (Section~\ref{secsnrt}), and furnishes supplementary details concerning the examples discussed  in the main manuscript (Section~\ref{secsuplinfoexam}).

\section{Additional details on our proposed approach.}\label{secadddetails}

The efficacy of our proposed approach is enhanced through a series of  refinements (implemented in all our examples), which are summarized below and detailed in  sections  \ref{bigsec4}, \ref{Secpseudocode} and \ref{secsnrt}.

\subsection{Ancestor pruning.}
As discussed earlier, rather than using a threshold on the signal-to-noise ratio to prune ancestors, we order the ancestors in decreasing contribution to the signal, the final number $q$ of ancestors is determined as the maximizer of noise to signal ratio increment $\frac{\V(n)}{\V(s)+\V(n)}(q+1)-\frac{\V(n)}{\V(s)+\V(n)}(q)$.

\begin{figure}[h]
    \centering
        \includegraphics[width=0.7\textwidth ]{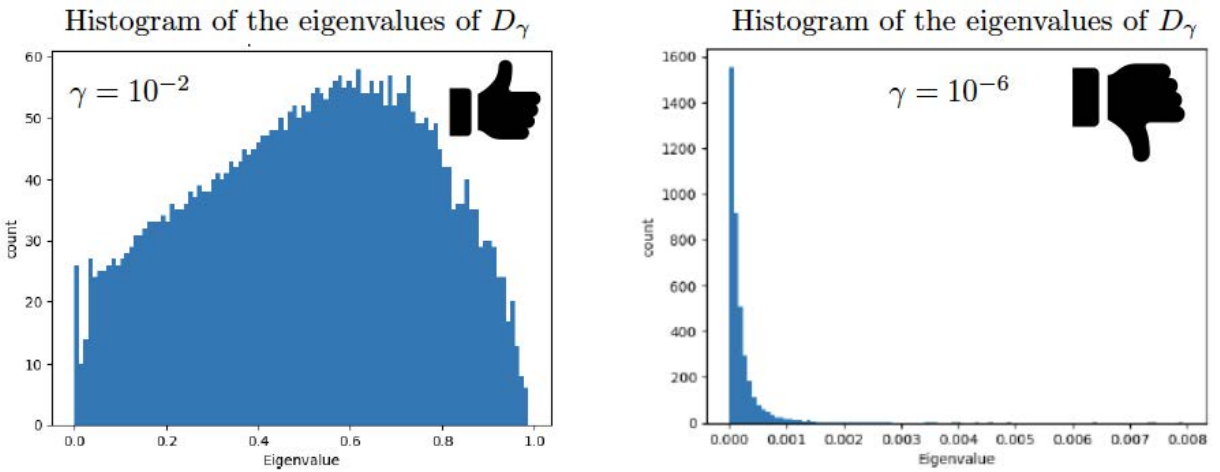}
    \caption{Histogram of the eigenvalues of $D_\gamma$=\eqref{eqkwkbejkejdbduAA} for $\gamma=10^{-2}$ (good choice) and $\gamma=10^{-6}$ (bad choice).}
    \label{fighistogram0}
\end{figure}

\subsection{Parameter Selection.}
The choice of the parameter $\gamma$ in \eqref{eqresidual} is a critical aspect of our proposed approach. We provide a structured approach for selecting $\gamma$ based on the characteristics of the kernel matrix $K_s$. Specifically, when $K_s$ 
is derived from a finite-dimensional feature map $\psi$ (i.e., when $K_s(x,x'):=\psi(x)^T  \psi(x')$ where the range of $\psi$ is finite-dimensional) and the data cannot be interpolated exactly with $K_s$ (the dimension of the range of   $\psi$ is smaller than the number of data points), we employ the regression residual to determine $\gamma$ as follows:
\begin{equation}
\gamma=\min_v\big\|v^T \psi (X)-Y\big\|_{\R^N}^2\,.
\end{equation}
Write $K_s(X,X)$ for the $N\times N$ matrix with entries $K_s(X_i,X_j)$.
Alternatively, when the data can be interpolated exactly with $K_s$ (e.g., when $K_s$ is a universal kernel), we select $\gamma$  (see Fig.~\ref{fighistogram0}) by maximizing  the variance of the eigenvalue histogram of the $N\times N$ matrix 
\begin{equation}\label{eqkwkbejkejdbduAA}
D_\gamma:=\gamma \big(K_s(X,X)+\gamma I\big)^{-1}\,,
\end{equation}
whose eigenvalues are bounded between $0$ and $1$ and converge towards $0$ as $\gamma \downarrow 0$ and towards $1$ as $\gamma \uparrow \infty$.
We can also select $\gamma$ as the median of the eigenvalues of $D_\gamma$.

 \subsection{Z-test quantiles.}
 The noise-to-signal ratio $\frac{\V(n)}{\V(s)+\V(n)}$ associated with  \eqref{eqresidual} admits the representer formula  $\frac{Y^T D_\gamma^2 Y}{Y^T D_\gamma Y}$. Therefore if the data is only comprised of noise (if $Y\sim \sigma^2 Z$ where $Z$ is a random vector with i.i.d. $\cN(0,1)$ entries), then the distribution of the noise-to-signal ratio follows that of the random variable
\begin{equation}
B:= \frac{Z^T D_\gamma^2 Z}{Z^T D_\gamma Z}\,.
\end{equation}
Therefore, the quantiles of $B$ can be used as an interval of confidence on the noise-to-signal ratio if $Y\sim \sigma^2 Z$.  Fig.~\ref{fignature3}.(c) shows these 
 $Z$-test quantiles  (in the absence of signal, the noise-to-signal ratio should fall within the shaded area with probability $0.9$).

\subsection{Generalizations on our proposed approach.}

\subsubsection{Complexity Reduction with Kernel PCA Variant.}
Write $K$ for the kernel associated with the RKHS $\H$ in Problem \ref{Pbkjedn}.
We use a variant of Kernel PCA \cite{mika1998kernel} to significantly reduces the computational complexity of our proposed method, making it primarily dependent on the number of principal nonlinear components in the kernel matrix $K(X,X)$ (the $N\times N$ matrix with entries $K(X_i,X_j)$) rather than the number of data points. 
To describe this write   $\lambda_1\geq \cdots \geq \lambda_r>0$ for the nonzero eigenvalues of $K(X,X)$ indexed in decreasing order and write
 $\alpha_{\cdot,i}$ for the corresponding unit-normalized eigenvectors, i.e. $K(X,X)\alpha_{\cdot,i}=\lambda_i \alpha_{\cdot,i}$. Then $|f(X)|^2=|f(\phi)|^2$, where $f(\phi)$ is the $r$  vector with entries $f(\phi_i):=\sum_{s=1}^N f(X_s)\alpha_{s,i}$.   Furthermore, writing $r'\leq r$ for the smallest index $i$ such that  $\lambda_i/\lambda_1<\epsilon$ where $\epsilon>0$ is some small threshold,
 the complexity of the problem can be further reduced (as in PCA) by  truncating $f(\phi)$ to $f(\phi')=(f(\phi_1),\ldots,f(\phi_{r'}))$  and approximating $\F$ with the space of functions $f\in \H$ such that  $|f(\phi')|^2\approx 0$.

\subsubsection{Generalizing Descendants and Ancestors with Kernel Mode Decomposition.}
We can extend the concept of descendants and ancestors to cover more complex functional dependencies between variables, including implicit ones. This generalization is achieved through a Kernel-based adaptation of Row Echelon Form Reduction (REFR), initially designed for affine systems, and leveraging the principles of Kernel Mode Decomposition \cite{owhadi2019kernelmd}.
To describe the connection with REFR consider the example in which  $\M$ is the manifold of $\R^3$ defined by the affine equations
$ x_1+  x_2+ 3 x_3-2=0$ and $x_1-x_2+x_3=0$, which is equivalent to selecting $\mathcal{F}=\Span\{f_1,f_2\}$ with $f_1(x)=x_1+  x_2+ 3 x_3-2$ and $f_2(x)=x_1-x_2+x_3$ in the problem formulation \ref{Pbkjedn}.
Then, irrespective of how we recover the manifold from data, the hypergraph representation of that manifold is equivalent to the row echelon form reduction of the affine system, and this representation and this reduction require a possibly arbitrary choice of free and dependent variables. So, for instance,  if we declare $x_3$ to be the free variables and $x_1$ and $x_2$ to be the dependent variables, then we can represent the manifold via the equations $x_1=1-2x_3$ and $x_2=1-x_3$ which have the  hypergraph representation depicted in Fig.~\ref{figmanifold1}.(b).
To describe the kernel generalization of REFR assume that the kernel $K$ can be decomposed as the additive kernel 
\begin{equation}
K=K_a+K_s+K_z\,,
\end{equation}
and write $\H_a$, $\H_s$, and $ \H_z$ for the RKHS induced by the kernels $K_a$, $K_s$, $K_z$.
Then a function $f\in \H$ can be decomposed as $f=f_a+f_s+f_z$ with $(f_a,f_s,f_z)\in \H_a\times \H_s\times \H_z$.
Then, generalizing REFR  we can approximate the manifold $\M$ via a manifold parametrized by
equations of the form 
\begin{equation}
f_a+f_s+f_z = 0 \Leftrightarrow g_a=f_s\,
\end{equation}
where  $f_a=-g_a$ and $g_a$ is a given function in $\H_a$ representing a dependent mode, 
 $f_z=0$  represents a zero mode, and $f_s\in \H_s$ is identified (regularized) as the  minimizer of the following variational problem
\begin{equation}
\min_{f_s\in \H_s}\|f_s\|_{K_s}^2+\frac{1}{\gamma} \big|(-g_a+f_s)(\phi)\big|^2\,.
\end{equation}
Taking $g_a(x)=x_1$ and $\H_s+\H_z$ to be a space of functions that does not depend on $x_1$ recovers our initial example \eqref{eqiheuedguydd} (with the pruning process encoded into the selection of $\H_z$).
 This generalization  is motivated by  its potential to recover implicit equations. For example, consider the implicit equation $x_1^2+x_2^2=1$, which can be retrieved by setting the mode of interest to be $g_a(x)=x_1^2$ and allowing $f_s$ to depend only on the variable $x_2$.

\section{Algorithm Overview for Type 3 problems: An Informal Summary}\label{secalgovertype3}

In this section, we provide an accessible overview of our algorithm's key components, which are further detailed in Algorithms \ref{alg1} and \ref{alg2} in Section \ref{Secpseudocode}. Our method focuses on determining the edges within a hypergraph. To achieve this, we consider each node individually, finding its ancestors and establishing edges from these ancestors to the node in question. While we present the algorithm for a single node, it can be applied iteratively to all nodes within the graph.
\bigbreak
\textbf{Algorithm for finding the ancestors of a node}:\begin{enumerate}
    \item \textbf{Initialization:} We start by assuming that all other nodes are potential ancestors of the current node.
    \item 
    \textbf{Selecting a Kernel}: We choose a kernel function, such as linear, quadratic, or fully nonlinear kernels (refer to Example \ref{exreg1}). The kernel selection process is analogous to the subsequent pruning steps, involving the determination of a parameter $\gamma$, regression analysis, and evaluation based on signal-to-noise ratios.
    \begin{itemize}
        \item \textbf{Kernel Selection Method:} 
         The choice of kernel follows a process similar to the subsequent pruning steps, including $\gamma$ selection, regression analysis, and signal-to-noise ratio evaluation.
        \item \textbf{Low Signal-to-Noise Ratio for All Kernels:}  If the signal-to-noise ratio is insufficient for all possible kernels, the algorithm terminates, indicating that the node has no ancestors.
    \end{itemize}
    \item \textbf{Pruning Process:}  While there are potential ancestors left to consider (details in Section \ref{seciterleastimp}):\begin{enumerate}
        \item \textbf{Identify the Least Important Ancestor:}  Ancestors are ranked based on their contribution to the signal (see Sec.~\ref{seckmdq}).
        \item \textbf{Noise prior:} Determine the value of $\gamma$ (see Section \ref{secselgamma}).
        \item \textbf{Regression Analysis:}  Predict the node's value using the current set of ancestors, excluding the least active one (i.e., the one contributing the least to the signal). We employ Kernel Ridge Regression with the selected kernel function and parameter $\gamma$ (see Sec.~\ref{secappchdgen} and \ref{secappchdpar}).
        \item \textbf{Evaluate Removal:} 
         Compute the regression signal-to-noise ratio (see Sec.~\ref{secstnratioint} and \ref{secsnrt}):
        \begin{itemize}
            \item \textbf{Low Signal-to-Noise Ratio:}  If the signal-to-noise ratio falls below a certain threshold, terminate the algorithm and return the current set of ancestors (see Section \ref{seciterleastimpaltsnrth}).
            \item\textbf{Adequate Signal-to-Noise Ratio:}  If the signal-to-noise ratio is sufficient, remove the least active ancestor and continue the pruning process.
        \end{itemize}
    \end{enumerate}
\end{enumerate}

\begin{figure}[h]
    \centering
        \includegraphics[width=1\textwidth ]{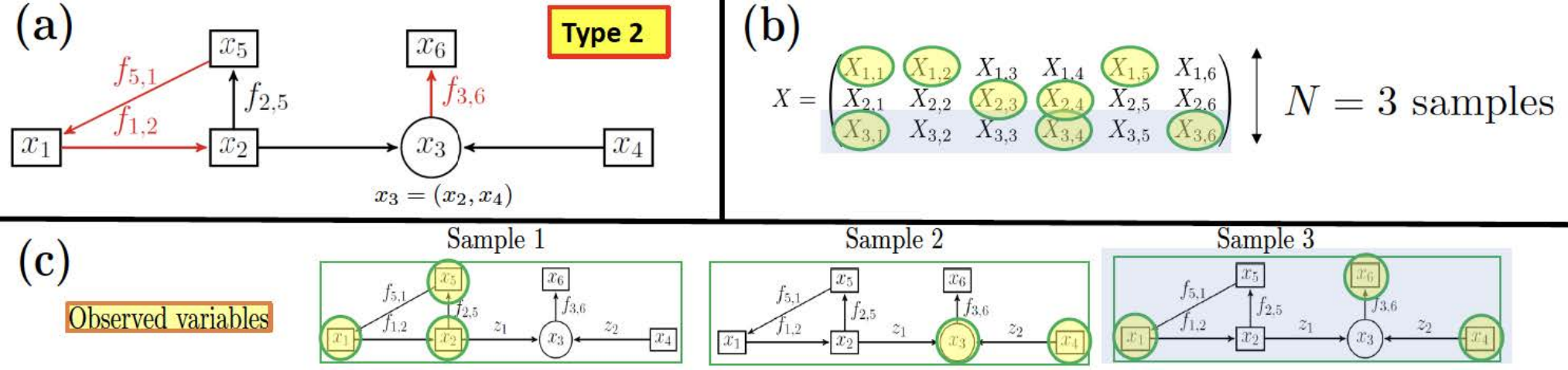}
    \caption{Formal description of Type 2 problems.}
    \label{figtype2}
\end{figure}

\section{Type 2 problems: Formal description and GP-based Computational Graph Completion}\label{sectype2formal}
\subsection{Formal description of Type 2 problems}\label{subsec:sectype2formal}
Consider a computational graph (as illustrated in Fig.~\ref{figtype2}.(a)) where nodes represent variables and edges are directed and they represent functions. These functions may be known or unknown. In Fig.~\ref{figtype2}.(a), edges associated with unknown functions ($f_{5,1}$, $f_{1,2}$, $f_{3,6}$) are colored in red, and those associated with known functions ($f_{2,5}$) are colored in black.
Round nodes are utilized to symbolize variables, which are derived from the concatenation of other variables (e.g, in Fig.~\ref{figtype2}.(a), $x_3=(x_2,x_4)$). Therefore, the underlying graph is, in fact, a hypergraph where functions may map groups of variables to other groups of variables, and we use round nodes to illustrate the grouping step. 
Given partial observations derived from $N$ samples of the graph's variables, we introduce a problem, termed a Type 2 problem, focused on approximating all unobserved variables and unknown functions.
Using Fig.~\ref{figtype2}.(a)-(b) as an illustration we call a vector $(X_{s,1},\ldots, X_{s,6})$ a sample from the graph if its entries are variables satisfying the functional dependencies imposed by the structure of the graph (i.e., $X_{s,1}=f_{5,1}(X_{s,5})$, $X_{s,2}=f_{1,2}(X_{2,s})$, $X_{s,3}=(X_{s,2}, X_{s,4})$, $X_{s,5}=f_{s,5}(X_{s,s})$, and $X_{s,6}=f_{3,6}(X_{s,3})$. These samples can be seen as the rows of given matrix $X$ illustrated in Fig.~\ref{figtype2}.(b) for $N=3$.
By partial observations, we mean that only a subset of the entries of each row may be observed, as illustrated in 
Fig.~\ref{figtype2}.(b)-(c). Note that a Type 2 problem combines a regression problem (approximating the unknown functions of the graph) with a matrix completion/data imputation problem (approximating the unobserved entries of the matrix $X$).

\subsection{Reminder on Computational Graph Completion for Type 2 problems}
Within the context of Sec.~\ref{subsec:sectype2formal}, the proposed GP solution to Type 2 problems is to simply replace unknown functions by  GPs and compute their Maximum A Posteriori (MAP)/Maximum Likelihood Estimation (MLE) estimators given available data and constraints imposed by the structure of the graph.
Taking into account the example depicted in Fig.~\ref{figtype2}, and substituting $f_{5,1}$, $f_{1,2}$, and $f_{3,6}$ with independent GPs, each with kernels $K, G$, and $\Gamma$ respectively, the objective of this MAP solution becomes minimizing  $ \|f_{5,1}\|_{K}^2+\|f_{1,2}\|_{G}^2+\|f_{3,6}\|_{\Gamma}^2$  (writing $\|f\|_K$ for the RKHS norm of $f$ induced by the kernel $K$) subject to the constraints imposed by the data and the functional dependencies encoded into the structure of the graph.

\begin{figure}[h]
    \centering
        \includegraphics[width=1\textwidth ]{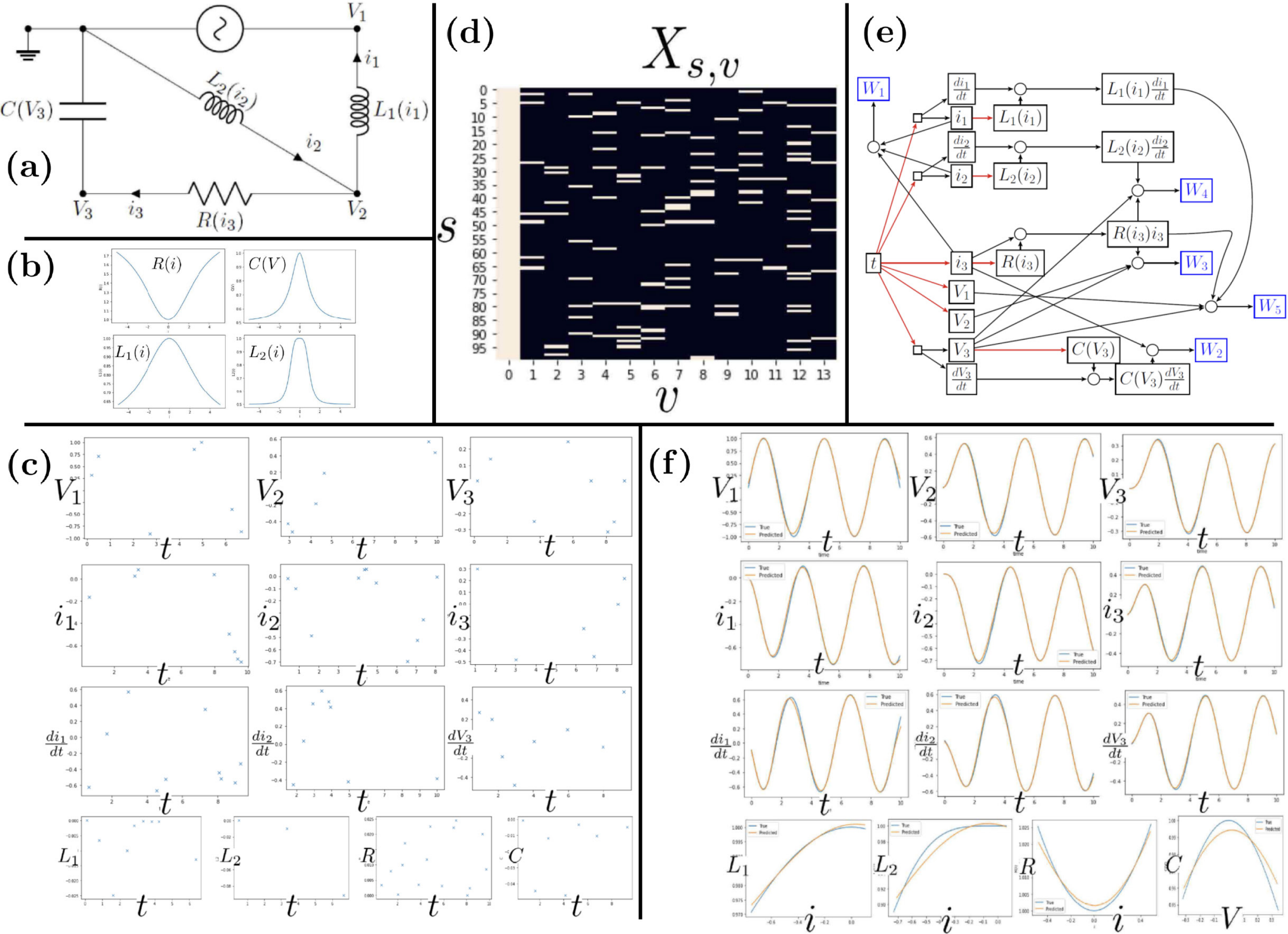}
    \caption{(a) Electric circuit. (b) Resistance, capacitance, and inductances are nonlinear functions of currents and voltages  (c) Measurements. (d) Kirchhoff's circuit laws. (e) The  computational graph with unknown functions represented as red edges. (f) Recovered functions.}
    \label{figelectriccircuit}
\end{figure}

\subsection{A system identification example.}
In order to exemplify Computational Graphical Completion (CGC), consider the system identification problem depicted in Fig.~\ref{figelectriccircuit}, sourced from \cite{owhadi2022computational}. Our objective is to identify a nonlinear electric circuit, as illustrated in Fig.~\ref{figelectriccircuit}.(a), from scarce measurement data. The nonlinearity of the circuit emanates from the resistance, capacitance, and inductances, which are nonlinear functions of currents and voltages, as shown in Fig.~\ref{figelectriccircuit}.(b).
Assuming these functions to be unknown, along with all currents and voltages as unknown time-dependent functions, we operate the circuit between times $0$ and $10$. Measurements of a subset of variables, representing the system's state, are taken at times $t_s=s/10$ for $s\in {0,\ldots,99}$.
Given these measurements, the challenge arises in approximating all unknown functions that define currents and voltages as time functions, capacitance as a voltage function, and inductances and resistance as current functions. Fig.~\ref{figelectriccircuit}.(c) displays the available measurements, which are notably sparse, preventing us from reconstructing the underlying unknown functions independently. Thus, their interdependencies must be utilized for approximation.
It is crucial to note that the system's state variables are interconnected through functional relations, as per Kirchhoff's laws for this nonlinear electric circuit, illustrated in Fig.~\ref{figelectriccircuit}.(d). These functional dependencies can be conceptualized as a computational graph, depicted in Fig.~\ref{figelectriccircuit}.(e), where nodes represent variables and directed edges represent functions. Known functions are colored in black, unknown functions in red, and round nodes aggregate variables, meaning edges map groups of variables, forming a hypergraph.
The CGC solution involves substituting the graph's unknown functions with Gaussian Processes (GPs), which may be independent or correlated, and then approximating the unknown functions with their Maximum A Posteriori (MAP) estimators, given the available data and the functional dependencies embedded in the graph's structure. Fig.~\ref{figelectriccircuit}.(f) showcases the true and recovered functions, demonstrating a notably accurate approximation despite the data's scarcity.

This simple example generalizes to an abstract framework detailed in \cite{owhadi2022computational}.  This framework has 
 a wide range of applications because most problems in CSE can also be formulated as completing computational graphs representing dependencies between functions and variables, and they can be solved in a similar manner by replacing unknown functions with GPs and by computing their MAP/EB estimator given the data. These problems include those illustrated in Fig.~\ref{fignature1}.(d-h).

\section{Hardness and well-posed formulation of Type 3 problems.}\label{sechardness}
In this subsection, we describe why Type 3 problems are challenging and why they can even be intractable if not formalized and approached properly.

\subsection{Curse of combinatorial complexity.}
First, the problem suffers from the curse of combinatorial complexity in the sense that the number of hypergraphs associated with $N$ nodes blows up rapidly with $N$. 
As an illustration, Fig.~\ref{figchd3v00} shows some of the hypergraphs associated with only three nodes.
A lower bound on that number is the  A003180 sequence, which answers the following question \cite{ishihara2001enumeration}: 
given $N$ unlabeled vertices, how many different hypergraphs in total can be realized on them by counting the equivalent hypergraphs only once?  
For $N=8$, this lower bound is  $\approx 2.78 \times 10^{73}$.

\begin{figure}[h]
    \centering
        \includegraphics[width=1\textwidth ]{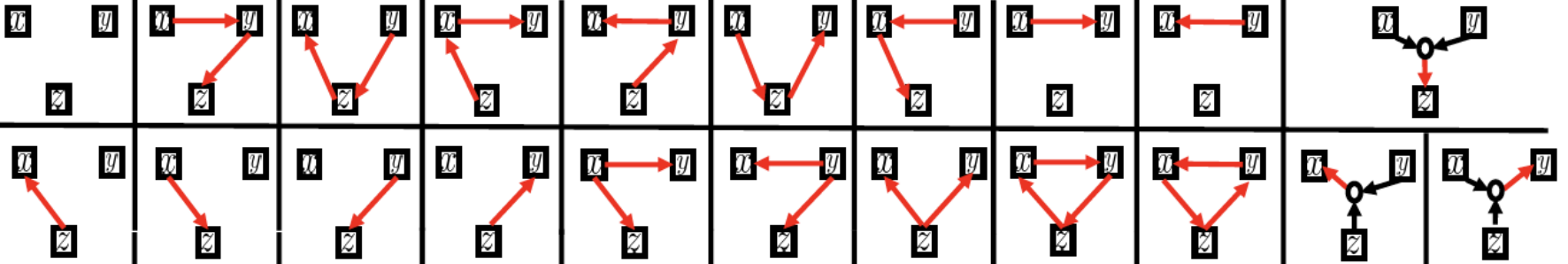}
    \caption{Computational Hypergraph Discovery with three variables }
    \label{figchd3v00}
\end{figure}

\subsection{Nonidentifiability and implicit dependencies.}
Secondly, it is important to note that, even with an infinite amount of data, the exact structure of the hypergraph might not be identifiable. To illustrate this point, let's consider a problem where we have $N$ samples from a computational graph with variables $x$ and $y$. The task is to determine the direction of functional dependency between $x$ and $y$. Does it go from $x$ to $y$ (represented as {\scalebox{0.4}{
\begin{tikzpicture}[->,>=stealth',shorten >=1pt,auto,node distance=3cm,
                    thick,main node/.style={rectangle,draw,font=\sffamily\Large\bfseries}]

\node[main node] (1) {$x$};
\node[main node] (2) [right of=1] {$y$};

\path[every node/.style={font=\sffamily\Large\bfseries}]
    (1) edge node [above ] {$f$} (2);

\end{tikzpicture}} ), or from $y$ to $x$ (represented as  {\scalebox{0.4}{
\begin{tikzpicture}[->,>=stealth',shorten >=1pt,auto,node distance=3cm,
                    thick,main node/.style={rectangle,draw,font=\sffamily\Large\bfseries}]

\node[main node] (1) {$y$};
\node[main node] (2) [right of=1] {$x$};

\path[every node/.style={font=\sffamily\Large\bfseries}]
    (1) edge node [above ] {$f$} (2);

\end{tikzpicture}})?

If we refer to Fig.~\ref{figambiguity}.(a), we can make a decision because $y$ can only be expressed as a function of $x$. In contrast, if we examine Fig.~\ref{figambiguity}.(b), the decision is also straightforward because $x$ can solely be written as a function of $y$. However, if the data mirrors the scenario in Fig.~\ref{figambiguity}.(c), it becomes challenging to decide as we can write both $y$ as a function of $x$ and $x$ as a function of $y$.
Further complicating matters is the possibility of implicit dependencies between variables. As illustrated in Fig.~\ref{figambiguity}.(d), there might be instances where neither $y$ can be derived as a function of $x$, nor $x$ can be represented as a function of $y$.

\begin{figure}[h]
    \centering
        \includegraphics[width=1\textwidth ]{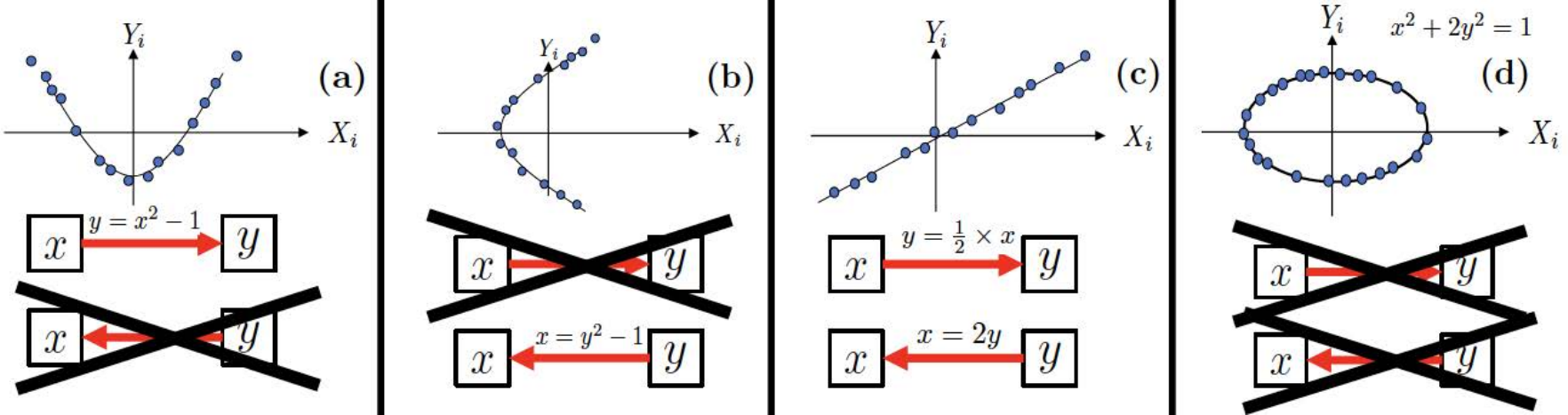}
    \caption{{\small The structure of the hypergraph is identifiable in (a), (b), and non-identifiable in (c). The relationship  between variables is implicit in (d).}}
    \label{figambiguity}
\end{figure}

\subsection{Causal inference and probabilistic graphs.}
Causal inference methods broadly consist of two approaches: constraint and score-based methods. While constraint-based approaches are asymptotically consistent, they only learn the graph up to an equivalence class~\cite{spirtes1991algorithm}. Instead, score-based methods resolve ambiguities in the graph's edges by evaluating the likelihood of the observed data for each graphical model. For instance, they may assign a higher evidence to $y\to x$ over $x\to y$ if the conditional distribution $x|y$ exhibits less complexity than $y|x$. The complexity of searching over all possible graphs, however, grows super-exponentially with the number of variables. Thus, it is often necessary to use approximate, but more tractable, search-based methods~\cite{chickering2002optimal, peters2017elements} or alternative criteria based on sensitivity analysis~\cite{data2016sensitivity}. For example, the preference could lean towards $y\to x$ rather than $x\to y$ if $y$ demonstrates less sensitivity to errors or perturbations in $x$.
In contrast, our proposed GP method  avoids the growth in complexity by performing a guided pruning process that assesses the contribution of each node to the signal. We also emphasize that our method 
is not limited to learning acyclic graph structures as it can identify feedback loops between variables.
Alternatively, methods for learning probabilistic undirected graphical models, also known as Markov networks, identify the graph structure by assuming the data is randomly drawn from some probability distribution~\cite{drton2017structure}. In this case, edges in the graph (or lack thereof) encode conditional dependencies between the nodes. A common approach learns the graph structure by modeling the data as being drawn from a multivariate Gaussian distribution with a sparse inverse covariance matrix, whose zero entries indicate pairwise conditional independencies~\cite{friedman2008sparse}. Recently, this approach has been extended using models for non-Gaussian distributions, e.g., in~\cite{baptista2021learning, ren2021learning}, as well as kernel-based conditional independence tests~\cite{zhang2011kernel}. In this work, we learn functional dependencies rather than causality or probabilistic dependence. We emphasize that we also do not assume the data is randomized 
or impose strong assumptions, such as additive noise models, in the data-generating process. 

We complete this paragraph by comparing the hypergraph discovery framework to structure learning for Bayesian networks and structural equation models (SEM). Let $x \in \R^d$ be a random variable with probability density function $p$ that follows the autoregressive factorization $p(x) = \prod_{i=1}^{d} p_i(x_i|x_1,\dots,x_{i-1})$ given a prescribed variable ordering. Structure learning for Bayesian networks aims to find the ancestors of variable $x_i$, often referred to as the set of parents $Pa(i) \subseteq \{1,\dots,i-1\}$, in the sense that $p_i(x_i|x_1,\dots,x_{i-1}) = p_i(x_i|x_{Pa(i)})$. Thus, the variable dependence of the conditional density $p_i$ is identified by finding the parent set so that $x_i$ is conditionally independent of all remaining preceding variables given its parents, i.e., $x_i \perp x_{1:i-1 \setminus Pa(i)} \vert x_{Pa(i)}$. Finding ancestors that satisfy this condition requires performing conditional independence tests, which are computationally expensive for general distributions~\cite{shah2020hardness}. Alternatively, SEMs assume that each variable $x_i$ is drawn as a function of its ancestors with additive noise, i.e, $x_i = f(x_{Pa(i)}) + \epsilon_i$ for some function $f$ and noise $\epsilon$~\cite{peters2017elements}. For Gaussian noise $\epsilon_i \sim \mathcal{N}(0,\sigma^2)$, each marginal conditional distribution in a Bayesian network is given by $p_i(x_i|x_{1:i-1}) \propto \exp(-\frac{1}{2\sigma^2}\|x_i - f(x_{1:i-1})\|^2)$. Thus, finding the parents for such a model by maximum likelihood estimation corresponds to finding the parents that minimize the expected mean-squared error $\|x_i - f(x_{Pa(i)})\|^2.$ Our approach minimizes a related objective, without imposing the strong probabilistic assumptions that are required in SEMs and Bayesian Networks. 
We also observe that while the graph structure identified in Bayesian networks is influenced by the specific sequence in which variables are arranged  (a concept exploited in numerical linear algebra \cite{SchaeferSullivanOwhadi17, schafer2021sparse} where Schur complementation is equivalent to conditioning GPs and a carefully ordering  leads to the accuracy of the Vecchia approximation $p_i(x_i|x_1,\dots,x_{i-1}) \approx  p_i(x_i|x_{i-k},\dots,x_{i-1})$ \cite{vecchia1988estimation}), the graph recovered by our approach remains unaffected by any predetermined ordering of those variables.

\begin{figure}[h]
    \centering
        \includegraphics[width=1\textwidth ]{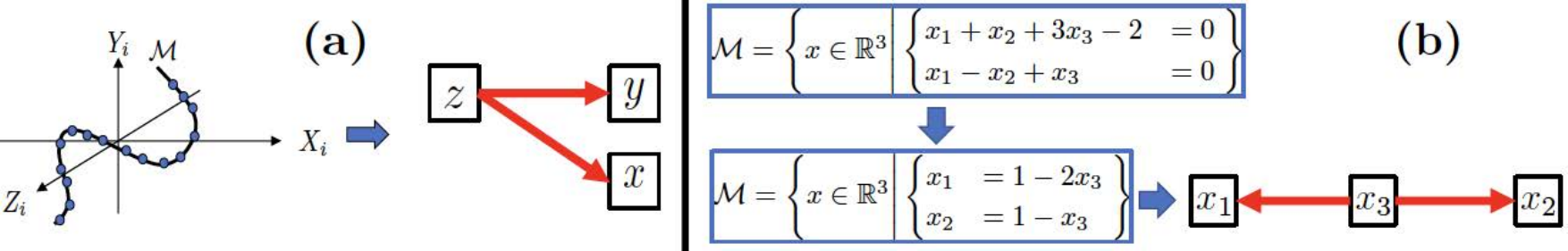}
    \caption{ (a) CHD formulation as a manifold discovery problem and hypergraph representation (b) The hypergraph representation of an affine manifold is equivalent to its Row Echelon Form Reduction.}
    \label{figmanifold1}
\end{figure}

\subsection{Well-posed formulation of the problem.}

In this paper, we focus on a formulation of the problem that remains well-posed even when the data is not randomized, i.e., we formulate the problem as the following manifold learning/discovery problem.

\begin{Problem}\label{Pbkjedn}
Let $\H$ be a Reproducing Kernel Hilbert Space (RKHS) of functions  mapping
$\R^d$ to $\R$. Let $\F$  be a closed linear subspace of $\H$ and let $\M$ be a subset of $\R^d$ such that $x\in \M$ if and only if  $f(x)=0$ for all $f\in \F$.
Given the (possibly noisy and nonrandom) observation of $N$ elements,  $X_1,\ldots,X_N$, of $\M$  approximate  $\M$.
\end{Problem}

To understand why problem \ref{Pbkjedn} serves as the appropriate formulation for hypergraph discovery, consider a manifold $\M\subset\R^d$. Suppose this manifold can be represented by a set of equations, expressed as a  collection of functions $(f_k)_k$ satisfying $\forall x\in\M,\ f_k(x)=0$.  To keep the problem tractable, we assume a certain level of regularity for these functions, necessitating they belong to a RKHS $\H$, ensuring the applicability of kernel methods for our framework. Given that any linear combination of the $f_k$ will also be evaluated to zero on $\M$, the relevant functions are those within the span of the $f_k$, forming a closed linear subspace of $\H$ denoted as $\F$.
The manifold $\M$ can be subsequently represented by a graph or hypergraph (see Fig.~\ref{figmanifold1}.(a)), whose ambiguity can be resolved through a deliberate decision to classify some variables as free and others as dependent. This selection could be arbitrary, informed by expert knowledge, or derived from probabilistic models or sensitivity analysis.

\section{A Gaussian Process method for Type 3 problems}\label{bigsec4}
\subsection{Affine case and Row Echelon Form Reduction.}\label{secaffine}
To describe the proposed solution to  Problem \ref{Pbkjedn}, we start with a simple example.  In this example $\H$ is a space of affine functions $f$ of the form 
\begin{equation}
f(x)=v^T \psi(x) \text{ with } \psi(x):=\begin{pmatrix}1\\x\end{pmatrix}\text{ and } v\in \R^{d+1}\,,.
\end{equation}
As a particular instantiation (see Fig.~\ref{figmanifold1}.(b)), we assume $\M$ to be the manifold of $\R^3$ ($d=3$) defined by the affine equations
\begin{equation}\label{eqhgvytfrt}
\M=\Bigg\{x\in \R^3 \Bigg|\begin{cases} x_1+  x_2+ 3 x_3-2&=0 \\x_1-x_2+x_3&=0\end{cases} \Bigg\}\,,
\end{equation}
which is equivalent to selecting $\mathcal{F}=\Span\{f_1,f_2\}$ with $f_1(x)=x_1+  x_2+ 3 x_3-2$ and $f_2(x)=x_1-x_2+x_3$ in the problem formulation \ref{Pbkjedn}.

Then, irrespective of how we recover the manifold from data, the hypergraph representation of that manifold is equivalent to the row echelon form reduction of the affine system, and this representation and this reduction require a possibly arbitrary choice of free and dependent variables. So, for instance, for the  system \eqref{eqhgvytfrt}, if we declare $x_3$ to be the free variables and $x_1$ and $x_2$ to be the dependent variables, then we can represent the manifold via the equations 
\begin{equation}
\M=\Bigg\{x\in \R^3 \Bigg|\begin{cases} x_1&=1-2x_3 \\ x_2&=1-x_3\end{cases} \Bigg\}\,,
\end{equation}
which have the  hypergraph representation depicted in Fig.~\ref{figmanifold1}.(b).

Now, in the $N>d$ regime where the number of data points is larger than the number of variables, the manifold can simply be approximated via a variant of PCA. Take $f^*\in\F$, we have $f^*(x)={v^*}^T\psi(x)$ for a certain $v^*\in\R^{d+1}$. Then for $X_s\in\M$, $f^*(X_s)=\psi(X_s)^Tv^*=0$. Defining 
\begin{equation}\label{eqjebiedbedb}
C_N:=\sum_{s=1}^N \psi(X_s)\psi(X_s)^T
\end{equation}
we see that $f^*(X_s)=0$ for all $X_s$ is equivalent to $C_Nv^*=0$. Since $N>d$, we can thus identify $\F$ exactly as $\{v^T\psi\text{ for }v\in Ker(C_N)\}$. We then obtain the manifold \begin{equation}\label{jehbdeudybeyb}
\M_N=\big\{x\in \R^d \mid v^T \psi(x)=0 \text{ for }v\in \operatorname{Span}(v_{r+1},\ldots,v_{d+1}) \big\}
\end{equation}
where $\operatorname{Span}(v_{r+1},\ldots,v_{d+1})$ is the zero-eigenspace of $C_N$. Here we write $\lambda_1\geq \cdots \geq \lambda_r >0= 
\lambda_{r+1}=\cdots=\lambda_{d+1}$ for the eigenvalues of $C_N$ (in decreasing order), and $v_1,\ldots,v_{d+1}$ for the corresponding eigenvectors 
 ($C_N v_i=\lambda_i v_i$).
 The proposed approach extends to the noisy case (when the data points are perturbations of elements of the manifold) by simply replacing
 the zero-eigenspace of the covariance matrix by the linear span of the eigenvectors associated with eigenvalues that are smaller than some  threshold $\epsilon>0$, i.e., by approximating  $\M$ with \eqref{jehbdeudybeyb} where $r$ is such that $\lambda_1\geq \cdots \geq \lambda_r \geq  \epsilon >
\lambda_{r+1}\geq \cdots \geq \lambda_{d+1}$.
In this affine setting \eqref{jehbdeudybeyb} allows us to estimate  $\M$ directly without RKHS norm minimization/regularization  because linear regression does not require regularization in the sufficiently large data regime. Furthermore the process of pruning ancestors  can be replaced by that of identifying sparse elements $v\in \operatorname{Span}(v_{r+1},\ldots,v_{d+1}) $ such that $v_i=1$.

\begin{figure}[h]
    \centering
        \includegraphics[width=0.7\textwidth ]{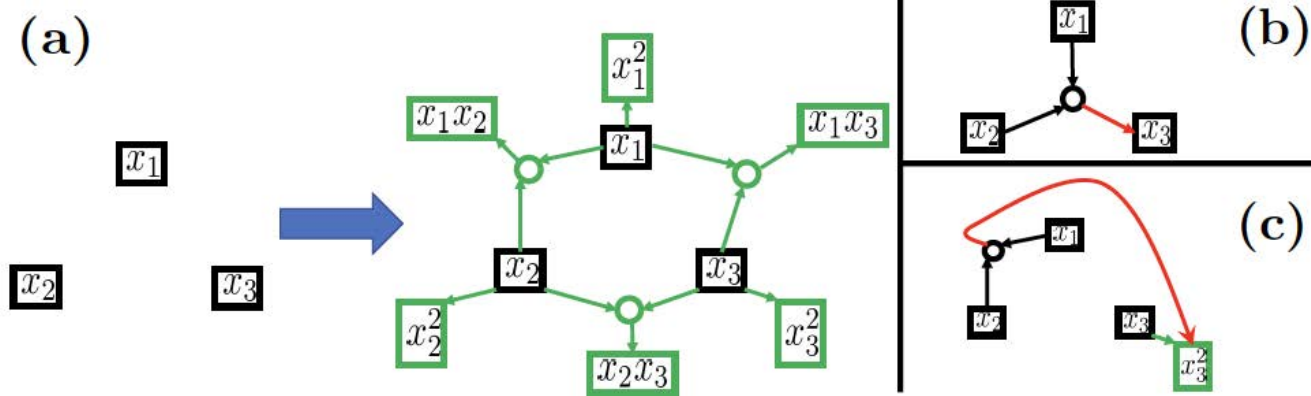}
    \caption{Feature map  generalization}
    \label{figchd10}
\end{figure}

\subsection{Feature map generalization.}\label{secfeaturemap}
This simple approach can be generalized by generalizing the underlying feature map $\psi$ used to define the space of functions (writing $d_\S$ for the dimension of the range of $\psi$)
\begin{equation}
\H=\big\{f(x)=v^T \psi(x)\mid v\in \R^{d_{\mathcal{S}}}\big\} \,.
\end{equation}
 For instance, if we use the feature  map 
 \begin{equation}\label{eqjhedhued}
\psi(x):=\big(1, \ldots,x_i,\ldots, x_i x_j, \ldots\big)^T\,
 \end{equation}
  then $\H$ becomes a space of quadratic polynomials on $\R^d$, i.e.,
  \begin{equation}
\H=\Big\{f(x)=v_0+\sum_{i} v_i x_i+\sum_{i\leq j} v_{i,j} x_i x_j\mid v\in \R^{d_\S}\Big\}\,,
\end{equation}
  and, in  the large data regime ($N>d_\S$), identifying quadratic dependencies between variables becomes equivalent to 
(1) adding nodes to the hypergraph corresponding to secondary variables obtained from primary variables $x_i$ through known functions (for \eqref{eqjhedhued}, these secondary variables are the quadratic monomials  $x_i x_j$, see Fig.~\ref{figchd10}.(a)), and (2)
identifying affine dependencies between the variables of the augmented hypergraph. The problem can, therefore, be reduced to the previous affine case.
Indeed, as in the affine case, the manifold can  then be approximated in the regime where the number of data points is larger than the dimension $d_\S$ of the feature map by 
\eqref{jehbdeudybeyb}, where $v_r,\ldots,v_N$ are the eigenvectors of $C_N=$\eqref{eqjebiedbedb} whose eigenvalues are zero (noiseless case) or smaller than some threshold $\epsilon>0$ (noisy case).

Furthermore, the hypergraph representation of the manifold is equivalent to a feature map generalization of Row Echelon Form Reduction to nonlinear systems of equations.
For instance, choosing $x_3$ as the dependent variable and $x_1,x_2$ as the free variables, $\mathcal{M}=\{x\in \R^3 \mid x_3-5x_1^2+x_2^2-x_1 x_2=0 \}$ can be represented as 
in Fig.~\ref{figchd10}.(b) where the round node represents the concatenated variable $(x_1,x_2)$ and the red arrow represents a quadratic function.
The generalization also enables the representation of implicit equations by selecting secondary variables as free variables. For instance, selecting $x_3^2$ as the free variable and  $x_1,x_2$ as the free variables, $\mathcal{M}=\{x\in \R^3 \mid x_1^2+ x_2^2+x_3^2-1=0 \}$ can be represented as 
in Fig.~\ref{figchd10}.(c).

\subsection{Kernel generalization and regularization.}\label{seckernelgen}

This feature-map extension of the previously discussed affine case  can evidently be generalized to arbitrary degree polynomials and to other basis functions.
However, as the dimension $d_\S$ of the range of the feature map $\psi$ increases beyond the number $N$ of data points, the problem becomes underdetermined:  the data only provides partial information about the manifold, i.e., it is not sufficient to uniquely determine the manifold. 
Furthermore, if the dimension of the feature map is infinite, then we are always in that low data regime, and we have the additional difficulty that we cannot directly compute with that feature map.
On the other hand, if $d_\S$ is finite (i.e., if the dictionary of basis functions is finite), then some elements of $\F$ (some constraints defining the manifold $\M$) may not be representable or well approximated as equations of the form $v^T \psi(x)=0$. To address these conflicting requirements, we need to kernelize and regularize the proposed approach (as done in interpolation).

\subsubsection{The kernel associated with the feature map.}\label{seckernelchdprem}
To describe this kernelization, we assume that the feature map $\psi$ maps $\R^d$ to some Hilbert space $\S$ that could be infinite-dimensional, and we write $K$ for the kernel defined by that feature map. To be precise, we now consider the setting where the feature map  $\psi$ is a function from $\R^d$ to a (possibly infinite-dimensional separable) Hilbert (feature) space $\S$ endowed with the inner product $\<\cdot,\cdot\>_\S$. To simplify notations, we will still write
$v^T w$ for $\<v,w\>_\S$ and $v w^T$ for the linear operator mapping $v'$ to $v \<w,v'\>_\S$.
Let 
\begin{equation}
\H:=\{v^T \psi(x)\mid v\in \S\}
\end{equation}
 be the space of functions mapping $\R^d$ to $\R$ defined by the feature map $\psi$.
To avoid ambiguity, assume (without loss of generality) that the identity
$v^T \psi(x)=w^T \psi(x)$ holds for all $x\in \R^d$ if and only if $v=w$. 
It follows that for $f\in \H$ there exists a unique $v\in \S$ such that $f=v^T \psi$. For $ f, g\in \H$ with $f=v^T \psi$ and $g=w^T \psi$, we can then define
\begin{equation}
\<f,g\>_\H:=v^T  w\,.
\end{equation}
Observe that $\H$ is a Hilbert space endowed with the inner product $\<\cdot,\cdot\>_\H$.
For $x,x'\in \X$, write
\begin{equation}\label{eqkerpsidef}
K(x,x'):=\psi(x)^T  \psi(x')\,,
\end{equation}
for the kernel defined by $\psi$ and observe that $(\H,\<\cdot,\cdot\>_\H)$ is the RKHS defined by the kernel $K$ (which is assumed to contain $\F$  in Problem \ref{Pbkjedn}). Observe in particular that for $f=v^T \psi \in \H$, $K$ satisfies the reproducing property
\begin{equation}\label{eqreproducingAAAQ}
\<f,K(x,\cdot)\>_\H=v^T  \psi(x)=f(x)\,.
\end{equation}

\subsubsection{Complexity Reduction with Kernel PCA Variant.}\label{secKPCAvariant}
We will now show that the previous feature-map PCA variant (characterizing the subspace of $f\in \H$ such that $f(X)=0$) can be kernelized as a variant of kernel PCA \cite{mika1998kernel}. 
To describe this write $K(X,X)$ for the $N\times N$ matrix with entries $K(X_i,X_j)$. Write  $\lambda_1\geq \lambda_2\geq \cdots \geq \lambda_{r}>0$ for the nonzero eigenvalues of $K(X,X)$ indexed in decreasing order and write
 $\alpha_{\cdot,i}$ for the corresponding unit-normalized eigenvectors, i.e.
\begin{equation}\label{eqkpca}
K(X,X)\alpha_{\cdot,i}=\lambda_i \alpha_{\cdot,i}\text{ and }|\alpha_{\cdot,i}|=1\,.
\end{equation}
Write $f(X)$ for the $N$ vector with entries $f(X_s)$.
For $i\leq r$, write 
\begin{equation}
\phi_i:= \sum_{s=1}^N \updelta_{X_s}\alpha_{s,i}
\end{equation}
 and 
 \begin{equation}
 f(\phi_i):=\sum_{s=1}^N f(X_s)\alpha_{s,i}\,.
 \end{equation}
 Write $f(\phi)$ for the $r$  vector with entries $f(\phi_i)$.

 Then, we have the following proposition.
 \begin{Proposition}
 The subspace  of  functions $f\in \H$ such that $f(\phi)=0$ is equal to the subspace of $f\in \H$ such that $f(X)=0$. Furthermore for 
 $f\in \H$ with feature map representation $f=v^T \psi$ with $v\in \S$ we have the identity (where $C_N=$\eqref{eqjebiedbedb})
 \begin{equation}\label{eqjgugyuuygbQ}
v^T C_N v=  |f(\phi)|^2=|f(X)|^2 \,.
\end{equation} 
 \end{Proposition}
 \begin{proof}
 Write  $\hat{\lambda}_1\geq \hat{\lambda}_2\geq \cdots \geq \hat{\lambda}_{\hat{r}}>0$ for the nonzero eigenvalues of $C_N=$\eqref{eqjebiedbedb} indexed in decreasing order. Write $v_1,\ldots,v_{r}$ for the corresponding eigenvectors, i.e.,
\begin{equation}\label{eqliwhdiuhedAAAQ}
C_N v_i =\hat{\lambda}_i v_i\,.
\end{equation}
Observing that 
\begin{equation}
C_N=\sum_{i=1}^r \hat{\lambda}_i v_i v_i^T
\end{equation}
 we deduce that the zero-eigenspace of $C_N$ is the set of vectors $v\in \S$ such that
$v^T v_i=0$ for $i=1,\ldots,r$.  
 Write $f_i:=v_i^T  \psi$. Observe that for $f=v^T \psi$, we have
$v_i^T v=\<f_i,f\>_K$.
Multiplying \eqref{eqliwhdiuhedAAAQ} by $\psi^T(x) $ implies
\begin{equation}\label{eqhwgdyged0AAAQ}
\sum_{s=1}^N K(x,X_s)f_i(X_s) = \hat{\lambda}_i f_i(x)
\end{equation}
\eqref{eqhwgdyged0AAAQ} implies that for $f=v^T \psi$
\begin{equation}\label{eqeiuhduedhAAAQ}
v_i^T v=\sum_{s=1}^N  \hat{\lambda}_i^{-1} f_i(X_s) \<K(\cdot,X_s),f\>_K=\sum_{s=1}^N \hat{\lambda}_i^{-1} f_i(X_s) f(X_s)
\end{equation}
where we have used the reproducing property \eqref{eqreproducingAAAQ} of $K$ in the last identity.
Write
\begin{equation}\label{eqlkesehdeudhqAAAQ}
\hat{\alpha}_{s,i}:=\lambda_i^{-1/2}f_i(X_s)\,.
\end{equation}
 Using \eqref{eqhwgdyged0AAAQ} with $x=X_{s'}$ implies that $\hat{\alpha}_{\cdot,i}$ is an eigenvector of the $N\times N$ matrix $K(X,X)$ with eigenvalue $\hat{\lambda}_i$. 
Taking $f=f_i$ in \eqref{eqeiuhduedhAAAQ} implies that $1=v_i^T v_i=|\hat{\alpha}_{\cdot,i}|^2$. Therefore, the $\hat{\alpha}_{\cdot,i}$ are unit-normalized. Summarizing, this analysis (closely related to the one found in kernel PCA \cite{mika1998kernel}) shows that the nonzero eigenvalues of $K(X,X)$ coincide with those of $C_N$ and we have $\hat{r}=r$, $\hat{\lambda_i}=\lambda_i$ and $\hat{\alpha}_{\cdot,i}=\alpha_{\cdot,i}$.
Furthermore, 
\eqref{eqeiuhduedhAAAQ} and \eqref{eqlkesehdeudhqAAAQ} imply  that for $i\leq r$, $v\in \S$ and $f=v^T \psi$, we have
\begin{equation}\label{eqjgugyuuygAAAQ}
v_i^T v= \lambda_i^{-1/2} f(X)\alpha_{\cdot,i}\,.
\end{equation}
 The identity \eqref{eqjgugyuuygAAAQ} then implies \eqref{eqjgugyuuygbQ}.
\end{proof}

\begin{Remark}
As in PCA the dimension/complexity of the problem can be further reduced by truncating $\phi$ to $\phi'=(\phi_1,\ldots,\phi_{r'})$ where $r'\leq r$ is identified as the smallest index $i$ such that  $\lambda_i/\lambda_1<\epsilon$ where $\epsilon>0$ is some small threshold.
\end{Remark}

\subsubsection{Kernel Mode Decomposition.}\label{seckmdq}
When the feature map $\psi$ is infinite-dimensional, the data only provides partial information about the constraints defining the manifold in the sense that  
$f(X)=0$ or equivalently $f(\phi)=0$ is a necessary but not sufficient condition for the zero level set of $f$ to be a valid constraint for the manifold (for $f$ to be such that $f(x)=0$ for all $x\in \M$). So we are faced with the following problems:  (1) How to regularize? (2) How do we identify free and dependent variables?
(3) How do we identify valid constraints for the manifold?
The proposed solution will be based on the Kernel Mode Decomposition (KMD)  framework introduced in \cite{owhadi2019kernelmd} (which shares conceptual foundations with Smoothing Spline ANOVA \cite{wahba2003introduction}).

\paragraph{\textbf{Reminder on KMD}}
We will now present a quick reminder on KMD in the setting of the following mode decomposition problem. So, in this problem, we have an unknown function $f^\dagger$ mapping some input space $\X$ to the real line $\R$. We assume that this function can be written as a sum of $m$ other unknown functions $f^\dagger_i$  which we will call modes, i.e.,
\begin{equation}
f^\dagger=\sum_{i=1}^m f^\dagger_i\,.
\end{equation}

We assume each mode $f^\dagger_i$ to be an unknown element of some RKHS $\H_{K_i}$ defined by some kernel $K_i$.
Then we consider the problem in which given the data $f^\dagger(X)=Y$ (with $(X,Y)\in \X^N \times \R^N$) we seek to approximate the $m$ modes composing the target function $f^\dagger$. Then, we have the following theorem. 

\begin{Theorem} \cite{owhadi2019kernelmd}
Using the relative  error in the  product norm $\|(f_1,\ldots,f_m)\|^2:=\sum_{i=1}^m \|f_i\|_{K_i}^2$ as a loss,
the minimax optimal recovery of $(f^\dagger_1,\ldots,f^\dagger_m)$ is $(f_1,\ldots,f_m)$ with
  \begin{equation}\label{eqiwdihdeuds}
f_i(x)=K_i(x,X) K(X,X)^{-1} Y\,,,
\end{equation}
where $K$ is the additive kernel
  \begin{equation}
K=\sum_{i=1}^m K_i\,.
\end{equation}
\end{Theorem}

The GP interpretation of this optimal recovery result is as follows. Let $\xi_i \sim \cN(0,K_i)$ be $m$ independent centered  GPs with kernels $K_i$. Write $\xi$ for the additive GP $\xi:=\sum_{i=1}^m \xi_i$. 
 \eqref{eqiwdihdeuds} can be recovered by replacing the modes $f_i^\dagger$ by independent centered GPs $\xi_i \sim \cN(0,K_i)$ with kernels $K_i$ and approximating the mode $i$ by conditioning $\xi_i$ on the available data $\xi(X)=Y$
where $\xi:=\sum_{i=1}^m \xi_i$ is the additive GP obtained by summing the independent GPs $\xi_i$, i.e.,
\begin{equation}
f_i(x)=\mathbb{E}\big[\xi_i(x)\mid \xi(X)=Y\big]\,.
\end{equation}
Furthermore 
 $(f_1,\ldots,f_m)$ can also be identified as the minimizer of
\begin{equation}\label{eqkjhbkyuyg6}
\begin{cases}\text{Minimize }&\sum_{i=1}^m \|f_i\|_{K_i}^2\\
\text{over }& (f_1,\ldots,f_m)\in  \mathcal{H}_{K_1}\times \cdots \times \mathcal{H}_{K_m}\\
\text{s. t. }&(\sum_{i=1}^m f_i)(X)=Y\,.\end{cases}\end{equation}
The variational formulation \eqref{eqkjhbkyuyg6} can be interpreted as a generalization of Tikhonov regularization which can be recovered by selecting $m=2$, $K_1$ to be a smoothing kernel (such as a Mat\'{e}rn kernel) and $K_2(x,y)=\sigma^2 \updelta(x-y)$ to be a white noise kernel.

Now, this abstract KMD approach  \cite{owhadi2019kernelmd} is associated with a quantification of how much each mode contributes to the overall data or how much each individual GP $\xi_i$ explains the data.
More precisely, the activation of the mode $i$ or GP $\xi_i$ can be quantified as 
\begin{equation}\label{eq:activation_formula}
p(i)=\frac{\|f_i\|_{K_i}^2}{\|f\|_K^2}\,,
\end{equation}
where $f=\sum_{i=1}^m f_i$. These activations $p(i)$ satisfy $p(i)\in [0,1]$ and $\sum_{i=1}^m p(i)=1$ they can be thought of as a generalization of Sobol sensitivity indices \cite{sobol2001global, sobol1993sensitivity, owen2013variance} to the nonlinear setting in the sense that they are associated with the following variance representation/decomposition  \cite{owhadi2019kernelmd} (writing $\<\cdot,\cdot\>_K$ for the RKHS inner product induced by $K$):
\begin{equation}
\operatorname{Var}\big[\<\xi,f\>_K\big] =\|f\|^2_K=\sum_{i=1}^m \|f_i\|_{K_i}^2=\sum_{i=1}^m \operatorname{Var}\big[\<\xi_i,f\>_K\big]
\end{equation}

\paragraph{\textbf{Application to CHD, general case.}}\label{secappchdgen}
Now, let us return to our original manifold approximation problem \ref{Pbkjedn} in the kernelized setting of \eqref{eqkerpsidef}.
Given the data $X$ we cannot regress an element $f\in \F$ directly since the minimizer of $\|f\|_K^2+\gamma^{-1}\|f(X)\|_{\R^N}^2$ is the null function. To identify the functions $f\in \F$, we need to decompose them into modes that can be interpreted as a generalization of the notion of free and dependent variables. To describe this, 
assume that the kernel $K$ can be decomposed as the additive kernel 
\begin{equation}
K=K_a+K_s+K_z. 
\end{equation}
Then $\H_K= \H_{K_a}+\H_{K_s}+\H_{K_z}$   implies that
for all function $f\in \H_K$, $f$ can be decomposed as $f=f_a+f_s+f_z$ with $(f_a,f_s,f_z)\in \H_a\times \H_s\times \H_z$.

\begin{Example}\label{exreg1}
As a running example, take $K$ to be the following additive kernel 
\begin{equation}\label{eqfullynonlinear}
K(x,x')=1+\beta_1 \sum_i x_i x_i' + \beta_2 \sum_{i\leq j} x_i x_j x_i' x_j' + \beta_3 \prod_i (1+k(x_i,x_i'))\,,
\end{equation}
that is the sum of a linear kernel, a quadratic kernel, and a fully nonlinear kernel. Take $K_a$ to be the part of the linear kernel that depends only on $x_1$, i.e.,
\begin{equation}
K_a(x,x')=\beta_1 x_1 x_1'\,.
\end{equation}
Take $K_s$ to be the part of the kernel that does not depend on $x_1$, i.e.,
\begin{equation}\label{eqldwdwod1}
K_s=1+\beta_1 \sum_{i\not=1} x_i x_i' + \beta_2 \sum_{i\leq j, i,j\not=1} x_i x_j x_i' x_j' + \beta_3 \prod_{i\not=1} (1+k(x_i,x_i'))\,.
\end{equation}
And take $K_z$ to be the remaining portion,
\begin{equation}
K_z=K-K_a-K_s\,.
\end{equation}
\end{Example}

Therefore the following questions are equivalent:
\begin{itemize} 
\item 	Given a function $g_a$ in the RKHS $\H_{K_a}$ defined by the kernel $K_a$ is there a function $f_s$ in the  RKHS $\H_{K_s}$ defined by the kernel $K_s$ such that $g_a(x)\approx f_s(x)$ for $x\in \M$? 
\item 	Given a function $g_a\in \H_{K_a}$  is there a function $f$ in the RKHS $\H_K$ defined by the kernel $K$ such that $f(x)\approx 0$ for $x\in \M$ and such that its
 $f_a$  mode is $-g_a$ and its $f_z$ mode is zero?
\end{itemize}

Then, the natural answer to the questions  is to 
identify the modes of the constraint  $f=f_a+f_s+f_z\in \mathcal{H}$ (such that $f(x)\approx 0$ for $x\in \M$ ) such that 
 $f_a=-g_a$ and $f_z=0$ by selecting $f_s$ to be the  minimizer of the following variational problem
\begin{equation}\label{eqjhjebdhbedd}
\min_{f_s\in \H_s}\|f_s\|_{K_s}^2+\frac{1}{\gamma} \big|(-g_a+f_s)(\phi)\big|^2\,.
\end{equation}
This is equivalent to introducing the additive GP $\xi=\xi_a+\xi_s+\xi_z+\xi_n$ whose modes are the independent GPs $\xi_a\sim \mathcal{N}(0,K_a)$, $\xi_s\sim \mathcal{N}(0,K_s)$, 
$\xi_z\sim \mathcal{N}(0,K_z)$, $\xi_n\sim \mathcal{N}(0,\gamma \updelta(x-y))$ (we use the label ``n'' in reference to ``noise''), and then recovering $f_s$ as 
\begin{equation}
f_s=\E\big[\xi_s\mid \xi(X)=0, \xi_a=-g_a, \xi_z=0\big]\,.
\end{equation}

\paragraph{\textbf{Application to CHD, particular case}.}\label{secappchdpar}
Taking $g_a(x)=x_1$ for our running example \ref{exreg1}, the previous questions are, as illustrated in Fig.~\ref{fignature2}(b), equivalent to asking whether there exists a function $f_s\in \H_{K_s}$ that does not depend on $x_1$ (since $K_s$ does not depend on $x_1$) such that\begin{equation}
x_1\approx f_s(x_2,\ldots,x_d) \text{ for } x\in \M\,.
\end{equation}
Therefore, the mode $f_a$ can be thought of as a dependent mode (we use the label ``a'' in reference to  ``ancestors''), the mode $f_s$ as a free mode (we use the label ``s'' in reference to ``signal''), the mode $f_z$ as a zero mode.

While our numerical illustrations have primarily focused on the scenario where $g_a$ takes the form of $g_a(x)=x_i$, and we aim to express $x_i$ as a function of other variables, the generality of our framework is motivated by  its potential to recover implicit equations. For example, consider the implicit equation $x_1^2+x_2^2=1$, which can be retrieved by setting the mode of interest to be $g_a(x)=x_1^2$ and allowing $f_s$ to depend only on the variable $x_2$.

\subsubsection{Signal-to-noise ratio.}\label{secstnratioint}
Now, we are led to the following question: since the mode $f_s$ (the minimizer of \eqref{eqjhjebdhbedd}) always exists and is always unique, how do we know that it leads to a valid constraint?
To answer that question, we compute the activation of the GPs used to regress the data. We write 
\begin{equation}
\V(s):=\|f_s\|_{K_s}^2\,,
\end{equation}
 for the activation of the signal GP $\xi_s$ and 
\begin{equation}
\V(n):=\frac{1}{\gamma} \big|(-g_a+f_s)(X)\big|^2
\end{equation}
for the activation of the noise GP $\xi_n$, and then these allow us to define a signal-to-noise ratio defined as
\begin{equation}\label{eqjhguyg65}
\frac{\V(s)}{\V(s)+\V(n)}\,.
\end{equation}
Note that this corresponds to activation ratio of the noise GP defined in (\ref{eq:activation_formula}). This ratio can then be used to test the validity of the constraint in the sense that if $V(s)/(V(s)+V(n))>\tau$  (with $\tau=0.5$ as a prototypical example), then the data is mostly explained by the signal GP and the constraint is valid.
If $V(s)/(V(s)+V(n))<\tau$, then the data is mostly explained by the noise GP and the constraint is not valid.

\subsubsection{Iterating by removing the least active modes from the signal.}\label{seciterleastimp}
If the constraint is valid, then we can next compute the activation of the modes composing the signal. To describe this, we assume that the kernel $K_s$ can be decomposed as the  additive kernel 
\begin{equation}
K_s=K_{s,1}+\cdots+K_{s,m}\,,
\end{equation}
 which results in $\H_{K_s}= \H_{K_{s,1}}+\cdots+\H_{K_{s,m}}$, 
which results in the fact that $\forall f_s \in \H_s$, $f_s$ can be decomposed as 
\begin{equation}
f_s= f_{s,1}+\cdots+f_{s,m}\,,
\end{equation}
 with 
$f_{s,i}\in \H_{K_{s,i}}$.
The activation of the mode $i$ can then be quantified as $p(i)=\|f_{s,i}\|_{K_{s,i}}^2/\|f_s\|_{K_s}^2$, which combined with 
$\|f_s\|_{K_s}^2=\sum_{i=1}^m \|f_{s,i}\|_{K_{s,i}}^2$ leads to $\sum_{i=1}^m p(i)=1$.

As our running example \ref{exreg1}, we can decompose $K_s=$\eqref{eqldwdwod1} as the sum of an affine kernel, a quadratic kernel, and a fully nonlinear kernel, i.e., $m=3$,
$K_{s,1}=1+\beta_1 \sum_{i\not=1} x_i x_i' $, $K_{s,2}=\beta_2 \sum_{i\leq j, i,j\not=1} x_i x_j x_i' x_j' $ and $K_{s,3}=\beta_3 \prod_{i\not=1} (1+k(x_i,x_i'))$.

As another example for our running example, we can take $K_s$ to be the sum of the portion of the kernel that does not depend on $x_1$ and $x_2$ and the remaining portion, i.e., $m=2$,
$K_{s,1}=1+\beta_1 \sum_{i\not=1,2} x_i x_i' + \beta_2 \sum_{i\leq j, i,j\not=1,2} x_i x_j x_i' x_j' + \beta_3 \prod_{i\not=1,2} (1+k(x_i,x_i'))$ and
$K_{s,2}=K_{s}-K_{s,1}$.

Then, we can order these sub-modes from most active to least active and create a new kernel $K_s$ by removing the least active modes from the signal and adding them to the mode that is set to be zero (see Fig.~\ref{figckmditeration}). To describe this, let $\pi(1),\cdots,\pi(m)$ be an ordering of the modes by their activation, i.e., $\|f_{s,\pi(1)}\|_{K_{s,\pi(1)}}^2 \geq \|f_{s,\pi(2)}\|_{K_{s,\pi(2)}}^2 \geq \cdots $.

Writing $K_t=\sum_{i=r+1}^m K_{s,\pi(i)}$ for the additive kernel obtained from the least active modes (with $r+1=m$ as the value used for our numerical implementations), we update the kernels $K_s$ and $K_z$ by assigning the least active modes from $K_s$ to $K_z$, i.e.,  $K_s-K_t \to K_s$ and $K_z+K_t \to K_z$ (we zero the least active modes).

\begin{figure}[h]
    \centering
        \includegraphics[width=0.7\textwidth ]{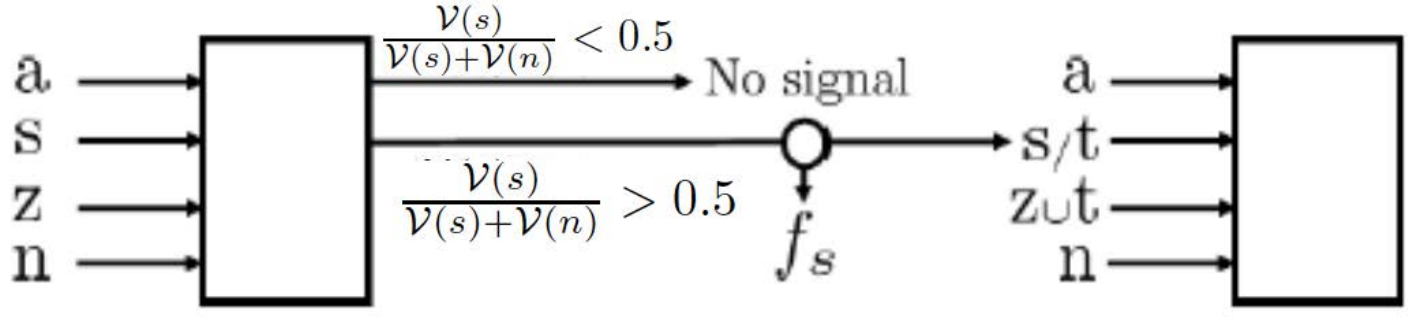}
    \caption{Iterating by removing the least active modes from the signal}
    \label{figckmditeration}
\end{figure}

Finally, we can iterate the process. This iteration can be thought of as identifying the structure of the hypergraph by placing too many hyperedges and removing them according to the activation of the underlying GPs. 

For our running example \ref{exreg1}, where we try to identify the ancestors of the variable $x_1$, if the sub-mode associated with the variable $x_2$ is found to be least active, then we can try to remove $x_2$ from the list of ancestors and try to identify $x_1$ as a function of $x_3$ to $x_d$. This is equivalent to selecting 
$K_a(x,x')=\beta_1 x_1 x_1'$,
\begin{equation}\label{eqjehbjedhbedudh}
K_{s/t}=1+\beta_1 \sum_{i\not=1,2} x_i x_i' + \beta_2 \sum_{i\leq j, i,j\not=1,2} x_i x_j x_i' x_j' + \beta_3 \prod_{i\not=1,2} (1+k(x_i,x_i'))\,,
\end{equation}
 and $K_{z\cup t}=K-K_a-K_{s/t}$ to assess whether there exists a function $f_s\in \H_K$ that does not depend on $x_1$ and $x_2$ s.t. $x_1\approx f_s(x_3,\ldots,x_d)$ for $x\in \M$.

\subsubsection{Alternative determination of the list of ancestors.}\label{seciterleastimpaltsnrth}
Our initial approach to determining the list of ancestors of a given node is to use a fixed threshold (e.g., $\tau=0.5$) to prune nodes.
 We propose a refined approach that mimics the strategy employed in Principal Component Analysis (PCA) for deciding which modes should be kept and which ones should be removed. The PCA approach is to order the modes in decreasing order of eigenvalues/variance and (1) either keep the smallest number modes holding/explaining  a given fraction (e.g., $90\%$) of the variance in the data, (2) or use an inflection point/sharp drop in the decay of the eigenvalues to select which modes should be kept.
Here, we propose a similar strategy. First we employ an  alternative determination of the least active mode: we iteratively remove the mode that leads to the smallest increase in noise-to-signal ratio, i.e., we remove the mode $t$ such that,
\begin{equation}
t=\operatorname{argmin_t}\frac{\V(n)}{\V(s/t)+\V(n)}\,.
\end{equation}
For our running example \ref{exreg1} in which we try to find the ancestors of the variable $x_1$ this is equivalent to removing the variables or node $t$ whose removal leads to the smallest loss in signal-to-noise ratio (or increase in noise-to-signal ratio) by selecting 
$$K_{s/t}=1+\beta_1 \sum_{i\not=1,t} x_i x_i' + \beta_2 \sum_{i\leq j, i,j\not=1,t} x_i x_j x_i' x_j' + \beta_3 \prod_{i\not=1,t} (1+k(x_i,x_i'))\,.$$
Next, we iterate this process, and we plot (a) the noise-to-signal ratio, and (b) the increase in noise-to-signal ratio as a function of the number of ancestors ordered according to this iteration.  Fig.~\ref{figfpu3} illustrates this process and shows that the removal of an essential node leads to a sharp spike in increase in the noise-to-signal ratio (the noise-to-signal ratio jumps from approximately 50-60\% to 99\%).
The identification of this inflection point can be used as a method for effectively and reliably pruning ancestors.

\begin{figure}
    \centering
    \includegraphics[width=\linewidth]{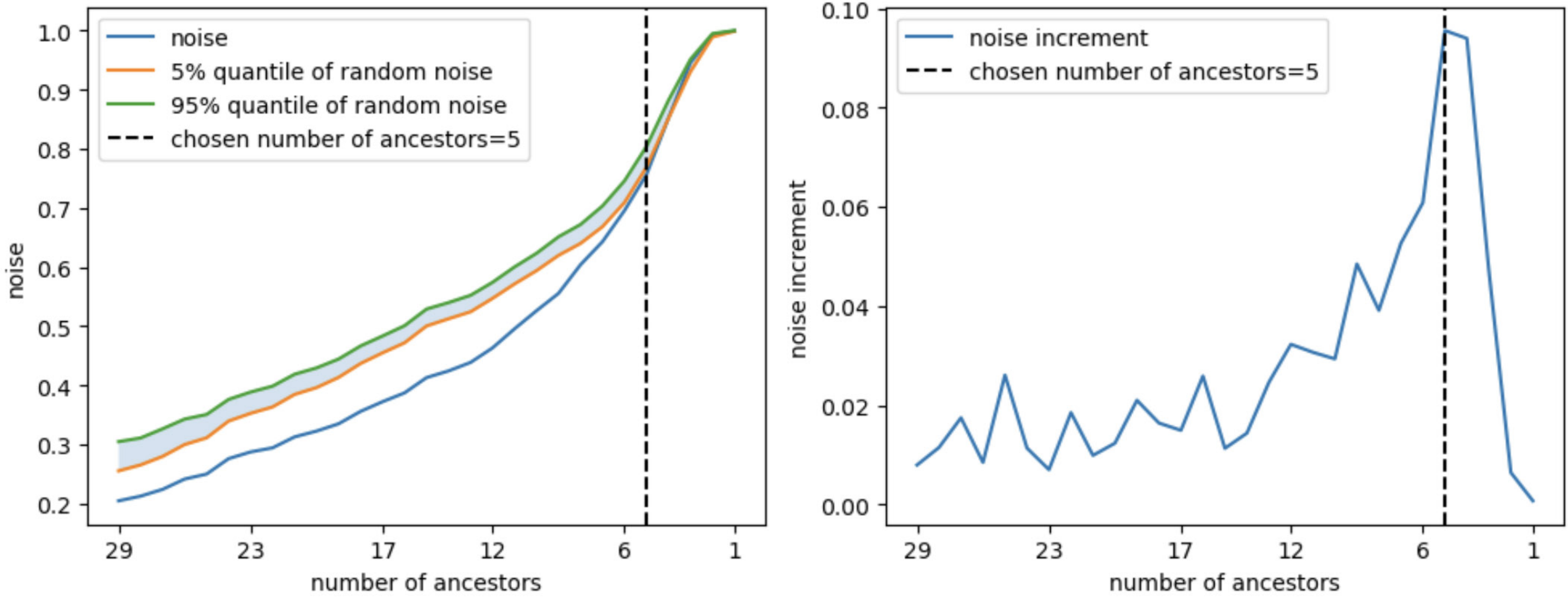}
    \caption{Computing the ancestors of the variable $\dot{x}_0$ in the Fermi-Pasta-Ulam-Tsingou problem. 
    (a) Noise-to-Signal Ratio, denoted as $\frac{\V(n)}{\V(s)+\V(n)}(q)$, with respect to the number of proposed ancestors, represented by $q$. Additionally, we include a visualization of the quantiles derived from the $Z$-test, as described in Section~\ref{secseltrzscore}. Notably, when there is no signal present, the noise-to-signal ratio is expected to fall within the shaded area with a probability of 0.9.
(b) Increments in the Noise-to-Signal Ratio, defined as $\frac{\V(n)}{\V(s)+\V(n)}(q)-\frac{\V(n)}{\V(s)+\V(n)}(q-1)$, as a function of the number of ancestors, denoted as $q$. The horizontal axis represents the number of proposed ancestors for $\dot{x}_0$. Determining an appropriate stopping point based solely on absolute noise-to-signal ratio levels can be challenging. In contrast, the increments in the noise-to-signal ratio clearly exhibit a discernible maximum, offering a practical point for decision-making.
    }
    \label{figfpu3}
\end{figure}

\begin{algorithm}
\caption{CHD by thresholding the signal-to-noise ratio }\label{alg1}
\begin{algorithmic}[1]
\Require Data $D$, set of nodes $V$, threshold $\tau$ ($\tau=0.5$ as a default value)
\Ensure Learned hypergraph \Comment{Set of ancestors for each node} 
\State  $D$ $\gets$ NormalizeData($D$) \Comment{Normalize the data} \label{alg1normalizedata}
\For{ $v\in V$}\label{alg1node}
\For{kernel $\in$ [\say{linear}, \say{quadratic}, \say{nonlinear}]} \Comment{Find the kernel}\label{alg1kernel}
\State SetOfAncestors($v$) $\gets$ All other nodes  
\State SignalToNoiseRatio $\gets$ ComputeSignalToNoiseRatio(kernel, node, $D$)  \label{alg1stnratio}
\If{SignalToNoiseRatio $>\tau$} choose that kernel and exit the for loop
\Else \, remove all ancestors from node
\EndIf
\EndFor
\While{SignalToNoiseRatio $>\tau$} \Comment{Prune ancestors}
\State Find least important ancestor \label{Linefindleastimanc1}
\State Recompute SignalToNoiseRatio without ancestor \label{Linefindleastimanc2}
\If{SignalToNoiseRatio $>\tau$} Remove that ancestor \label{Linefindleastimanc3}
\EndIf
\EndWhile
\EndFor
\end{algorithmic}
\end{algorithm}

\section{Algorithm pseudocode.}\label{Secpseudocode}

Our overall method is summarized in the pseudocode Alg.~\ref{alg1} and Alg.~\ref{alg2} that we will now describe. 
Alg.~\ref{alg1} takes the data $D$ (encoded into the samples $X_1,\ldots,X_N$ of Problem \ref{Pbkjedn}) 
and the set of nodes $V$ as an input and produces, as described in Sec.~\ref{seckernelgen}, 
for each node $i\in V$ its set of minimal ancestors $A_i$ and the simplest possible function $f_i$ such that $x_i\approx f_i\big((x_j)_{j\in A_i}\big)$. It employs the default threshold of $0.5$ on the signal-to-noise ratios for its operations.
Line \ref{alg1normalizedata} normalizes the data (via an affine transformation) so that the samples $X_i$ are of mean zero and variance $1$.
Given a node with index $i=1$ in Line \ref{alg1node} ($i$ runs through the set of nodes, and we select $i=1$ for ease of presentation), the command in Line \ref{alg1kernel} refers to selecting a signal kernel of the form $K_s=$\eqref{eqldwdwod1} (where $k$ is selected to be a vanilla RBF kernel such as Gaussian or Mat\'{e}rn), 
 with 
$1\geq \beta_1 >0=\beta_2=\beta_3$ for the linear kernel, $1\geq \beta_1 \geq \beta_2>0=\beta_3$ for the quadratic kernel and $1\geq \beta_1 \geq \beta_2\geq \beta_3>0$ for the fully nonlinear (interpolative) kernel.
 The ComputeSignalToNoiseRatio function in Line \ref{alg1stnratio} computes the signal-to-noise ratio with $g_a(x)=x_1$ and with the kernel selected in Line \ref{alg1kernel}.
The value of $\gamma$ is selected automatically by maximizing the variance of the histogram of eigenvalues of   $D_\gamma$  as described in Sec.~\ref{secselgamma} 
 (with the kernel $K=K_s=$\eqref{eqldwdwod1} selected in Line \ref{alg1kernel} and $Y=g_a(X)$ with $g_a(x)=x_1$).
The value of $\gamma$ is re-computed whenever a node is removed from the list of ancestors, and $K_s$ is nonlinear.
 Lines \ref{Linefindleastimanc1}, \ref{Linefindleastimanc2} and \ref{Linefindleastimanc3} are described in Sec.~\ref{seciterleastimp}. They correspond to iteratively identifying the ancestor node $t$ contributing the least to the signal and removing that node from the set of ancestors of the node $1$ if the removal of that node $t$ does not send the signal-to-noise ratio below the default threshold $0.5$.
 
\begin{algorithm}[!ht]
\caption{CHD by inflection point in the noise-to-signal ratio }\label{alg2}
\begin{algorithmic}[1]
\Require Data $D$, set of nodes $V$, threshold $\tau$ ($\tau=0.5$ as a default value)
\Ensure Learned hypergraph \Comment{Set of ancestors for each node} 
\State  $D$ $\gets$ NormalizeData($D$) \Comment{Normalize the data} 
\For{node $v\in V$}\label{alg1node_inflection}
\For{kernel $\in$ [\say{linear}, \say{quadratic}, \say{nonlinear}]} \Comment{Find the kernel}
\State SetOfAncestors $\gets$ All other nodes  
\State SignalToNoiseRatio $\gets$ ComputeSignalToNoiseRatio(kernel, node, $D$)  
\If{SignalToNoiseRatio $>\tau$} choose that kernel and exit the for loop
\Else \, remove all ancestors from node
\EndIf
\EndFor
\State q $\gets$ Cardinal(All other nodes) 
\State SetOfAncestors$(q)$ $\gets$ All other nodes  
\While{q $\geq 1$} 
\State  NoiseToSignalRatio($q$)  $\gets$ ComputeNoiseToSignalRatio(kernel, node, $D$)  
\State LeastImportantAncestor $\gets$ Find least important ancestor in SetOfAncestors($q$)
\State SetOfAncestors$(q-1)$  $\gets$ SetOfAncestors($q$) $\setminus$ LeastImportantAncestor
\State $q$ $\gets$ $q-1$
\EndWhile
\State $q^\dagger$ $\gets$ Inflection point in ($q\rightarrow$ NoiseToSignalRatio$(q)$) or  spike in ($q\rightarrow$ NoiseToSignalRatio$(q)$ - NoiseToSignalRatio$(q-1)$)
\State FinalSetOfAncestors($v$) $\gets$ SetOfAncestors$(q^\dagger)$ 
\EndFor
\end{algorithmic}
\end{algorithm}

Algorithm \ref{alg2} distinguishes itself from Algorithm \ref{alg1} in its approach to pruning ancestors based on signal-to-noise ratios. Instead of using a default threshold of $0.5$ like Algorithm \ref{alg1}, Algorithm \ref{alg2} computes the noise-to-signal ratio, represented as $\frac{\V(n)}{\V(s)+\V(n)}(q)$. This ratio is calculated as a function of the number $q$ of ancestors, which are ordered based on their decreasing contribution to the signal. The detailed methodology behind this computation can be found in Section \ref{seciterleastimpaltsnrth} and is visually depicted in Figure \ref{figfpu3}. The final number $q$ of ancestors is then determined by finding the value that maximizes the difference between successive noise-to-signal ratios, $\frac{\V(n)}{\V(s)+\V(n)}(q+1)-\frac{\V(n)}{\V(s)+\V(n)}(q)$.

\section{Analysis of the signal-to-noise ratio test.}\label{secsnrt}
\subsection{The signal-to-noise ratio depends on the prior on the level of noise.}
The signal-to-noise ratio \eqref{eqjhguyg65} depends on the value of $\gamma$, which is the variance prior on the level of noise. The goal of this subsection is to answer the following two questions: (1)	How do we select $\gamma$? (2) How do we obtain a confidence level for the presence of a signal? Or equivalently for a hyperedge of the hypergraph?
To answer these questions, we will now analyze the signal-to-noise ratio in the following regression problem in which we seek to approximate the unknown function $f^\dagger\,:\,  \X \to \R$ based on noisy observations 
\begin{equation}\label{eqlkjdendhd2A}
f^\dagger(X)+\sigma Z=Y
\end{equation}
 of its values at collocation points $X_i$ ($(X,Y)\in \mathcal{X}^N\times \R^N$, $Z\in \R^N$,  and the entries $Z_i$ of $Z$ are i.i.d  $\N(0,1 )$).
Assuming $\sigma^2$ to be unknown and writing $\gamma$ for a candidate for its value,
recall that the GP solution to this problem is  approximate $f^\dagger$ by interpolating the data with the sum of two independent GPs, i.e., 
\begin{equation}
f(x)=\mathbb{E}[\xi(x)|\xi(X)+\sqrt{\gamma} Z=Y]\,,
\end{equation}
where $\xi \sim \cN(0,K)$ is the GP prior for the signal $f^\dagger$ and $\sqrt{\gamma} Z \sim \mathcal{N}(0,\gamma I_N)$ is the GP prior for  the noise $\sigma Z$ in the measurements.
Following Sec.~\ref{seckmdq} $f$ can also be identified as a minimizer of
\begin{equation}
\text{minimize}_{f'}\|f'\|_{K}^2+\frac{1}{\gamma}\|f'(X)-Y\|^2_{\R^N}\,,
\end{equation}
the activation of the signal GP can be quantified as $s=\|f\|_{K}^2$, the activation of the noise GP can be quantified as 
$\V(n)=\frac{1}{\gamma}\|f(X)-Y\|_{\R^N}^2$.
We can then define the  noise to signal ratio $\frac{\V(n)}{\V(s)+\V(n)}$, which  admits the following representer formula,
\begin{equation}\label{eqkkjedjedkd}
\frac{\V(n)}{\V(s)+\V(n)}=\gamma \frac{Y^T \big(K(X,X)+\gamma I\big)^{-2} Y}{Y^T \big(K(X,X)+\gamma I\big)^{-1} Y}\,.
\end{equation}
Observe that when applied to the setting of Sec.~\ref{secstnratioint}, this signal-to-noise ratio is calculated with $K=K_s$ and $Y=g_a(X)$. 

Now we have the following proposition, which follows from \eqref{eqkkjedjedkd}.
\begin{Proposition}
It holds true that $\frac{\V(n)}{\V(s)+\V(n)} \in [0,1]$, and if $K(X,X)$ has full rank, 
\begin{equation}
\lim_{\gamma \downarrow 0} \frac{\V(n)}{\V(s)+\V(n)}=0 \text{ and } \lim_{\gamma \uparrow \infty} \frac{\V(n)}{\V(s)+\V(n)}=1\,.
\end{equation}
\end{Proposition}
Therefore, we are led to the following question: if the signal $f^\dagger$ and the level of noise $\sigma^2$ are both unknown, how do we select $\gamma$ to decide whether the data is mostly signal or noise?

\subsection{How do we select the prior on the level of noise?}\label{secselgamma}
Our answer to this question depends on whether the feature-map associated with the base kernel $K$ is finite-dimensional or not.

\subsubsection{When the kernel is linear, quadratic or associated with a finite-dimensional feature map.}\label{secgamafindimpsi}
If the feature-map associated with the base kernel $K$ is finite-dimensional, then $\gamma$ can be estimated from the data itself when the number of data-points is sufficiently large (at least larger than the dimension of the feature-space $\S$).
A prototypical example (when trying to identify the ancestors of the variable $x_1$) is $K=K_s$=\eqref{eqldwdwod1} with $\beta_3=0$.
In the general setting assume that $K(x,x'):=\psi(x)^T  \psi(x')$ where the range $\S$ of $\psi$ is finite-dimensional.
Assume that $f^\dagger$ belongs to the RKHS defined by $\psi$, i.e., assume that it is of the form $f^\dagger= v^T \psi$ for some $v$ in the feature-space.
Then \eqref{eqlkjdendhd2A} reduces to
\begin{equation}\label{eqlkjdendhd2AB}
v^T \psi(X)+\sigma Z=Y\,,
\end{equation}
and, in the large data regime, $\sigma^2$ can be estimated by
\begin{equation}\label{equeddudydddy}
\bar{\sigma}^2:= \frac{1}{N}\inf_{w \in \S} \big\|w^T \psi(X)-Y\big\|_{\R^N}^2\,.
\end{equation}
Our strategy, when the feature map is finite-dimensional, is then to select 
\begin{equation}\label{eq:choice_gamma}
\gamma= N \bar{\sigma}^2= \inf_{w \in \S} \big\|w^T \psi(X)-Y\big\|_{\R^N}^2\,.
\end{equation}

\subsubsection{When the kernel is interpolatory (associated with an infinite-dimensional feature map).}\label{secfullynonlinearkerngam}
If the feature-map associated with the base kernel $K$ is infinite-dimensional (or has more dimensions than we have data points) then it can interpolate the data exactly and the previous strategy cannot be employed since the minimum of \eqref{equeddudydddy} is zero. A prototypical example (when trying to identify the ancestors of the variable $x_1$) is $K=K_s$=\eqref{eqldwdwod1} with $\beta_3>0$.
In this situation, we do not attempt to estimate the level of noise $\sigma$ but select a prior $\gamma$ such that the resulting noise-to-signal ratio can effectively differentiate noise from signal. To describe this, observe that the noise-to-signal ratio \eqref{eqkkjedjedkd} admits the  representer formula 
\begin{equation}\label{eqhejbjebdd}
\frac{\V(n)}{\V(s)+\V(n)}= \frac{Y^T D_\gamma^2 Y}{Y^T D_\gamma Y}\,,
\end{equation}
involving the $N\times N$ matrix
\begin{equation}\label{eqkwkbejkejdbdu}
 D_\gamma:=\gamma  \big(K(X,X)+\gamma I\big)^{-1}\,.
\end{equation}
Observe that $0\leq D_\gamma \leq I$,  and
\begin{equation}
\lim_{\gamma \downarrow 0} D_\gamma=0 \text{ and } \lim_{\gamma \uparrow \infty} D_\gamma=I\,.
\end{equation}

Write  $(\lambda_i,e_i)$ for the eigenpairs of $K(X,X)$ ($K(X,X) e_i=\lambda_i e_i$) where the $\lambda_i$ are ordered in decreasing order. 
Then the eigenpairs of $D_\gamma$ are $(\omega_i,e_i)$ where
\begin{equation}
\omega_i:=\frac{\gamma}{\gamma+\lambda_i}\,.
\end{equation}
Note that the  $\omega_i$ are contained in $[0,1]$ and also ordered in decreasing order. 

Writing $\bar{Y}_i$ for the orthogonal projection of  $Y$ onto $e_i$, we have
\begin{equation}\label{eq:explicit_signal_to_noise_ratio}
\frac{\V(n)}{\V(s)+\V(n)}= \frac{\sum_{i=1}^n \omega_i^2 \bar{Y}_i^2 }{\sum_{i=1}^n \omega_i \bar{Y}_i^2}\,,
\end{equation}
It follows that if the histogram of the eigenvalues of $D_\gamma$ is concentrated near $0$ or near $1$, then the noise-to-signal ratio is non-informative since the prior $\gamma$ dominates it. To avoid this phenomenon, we select $\gamma$ so that the eigenvalues of $D_\gamma$ are well spread out in the sense that the histogram of its eigenvalues has maximum or near-maximum variance (see Fig.~\ref{fighistogram0} for a good choice and a bad choice for $\gamma$). If the eigenvalues have an algebraic decay, then this is equivalent to taking $\gamma$ to be the geometric mean of those eigenvalues.\\
In practice, we use an off-the-shelf optimizer to obtain $\gamma$ by maximizing the sample variance of $(\omega_i)_{i=1}^n$. If this optimization fails, we default to the median of the eigenvalues. This ensures a balanced, well-spread spectrum for $D\gamma$, with half of the eigenvalues $\lambda_i$ being lower and half being higher than the median.

\subsubsection{Rationale for the choices of $\gamma$} 

The purpose of this section is to present a rationale for the proposed choices for $\gamma$ in Sec.~\ref{secgamafindimpsi} and \ref{secfullynonlinearkerngam}. 
For the choice Sec.~\ref{secgamafindimpsi}, we present an asymptotic analysis of the signal-to-noise ratio in the setting of a simple linear regression problem.  According to \eqref{eq:choice_gamma}, $\gamma$ must scale linearly in $N$; this scaling is necessary to achieve a ratio that represents the signal-to-noise per sample. Without it (if $\gamma$ remains bounded as a function of $N$), this scaling of the signal-to-noise would converge towards $0$ as $N\rightarrow \infty$. To see how we will now consider a  simple example in which we seek to linearly regress the variable $y$ as a function of the variable $x$, both taken to be scalar (in which case $\psi(x)=x$). Assume that the samples are of the form $Y_i=a X_i+\sigma Z_i$ for $i=1,\ldots,N$, where  $a, \sigma\not=0$, the $Z_i$ are i.i.d. $\cN(0,1)$ random variables, and the $X_i$ satisfy $\frac{1}{N}\sum_{i=1}^N X_i=0$ and $\frac{1}{N}\sum_{i=1}^N X_i^2=1$. Then, the signal-to-noise ratio is $\frac{\V(s)}{\V(s)+\V(n)}$ with $\V(s)=|v|^2$ and $\V(n)=\frac{1}{\gamma}\sum_{i=1}^N |v X_i-Y_i|^2$ and $v$ is a minimizer of 
\begin{equation}
\min_{v\in \R}|v|^2+\frac{1}{\gamma}\sum_{i=1}^N |v X_i-Y_i|^2\,.
\end{equation}
In asymptotic $N\rightarrow \infty$ regime, we have $v\approx \frac{aN}{\gamma +N}$ and 
\begin{equation}\label{eqkdeidhdiepid}
\frac{\V(s)}{\V(s)+\V(n)}\approx \frac{\frac{\gamma}{N} a^2}{- a^2 (\gamma/N+1)+(a^2+\sigma^2) (\gamma/N+1)^2}\,.
\end{equation}
If $\gamma$ is bounded independently from $N$, then $\frac{\V(s)}{\V(s)+\V(n)}$ converges towards zero as $N\rightarrow \infty$, which is undesirable as it does not represent a signal-to-noise ratio per sample.  If $\gamma=N$, then 
$\frac{\V(s)}{\V(s)+\V(n)}\approx \frac{a^2}{4\sigma^2+2 a^2}$, which does not converge to $1$ as $a\rightarrow \infty$ and $\sigma \rightarrow 0$, which is also undesirable.
If $\gamma$ is taken as in \eqref{eq:choice_gamma}, then $\gamma \approx N \sigma^2$ and 
\begin{equation}
\frac{\V(s)}{\V(s)+\V(n)}\approx \frac{a^2}{ (\sigma^2+1)(a^2+\sigma^2+1)}\,,
\end{equation}
which converges towards $0$ as $\sigma \rightarrow \infty$ and towards $1/(1+\sigma^2)$ as $a\rightarrow \infty$, which has, therefore, the desired properties.

Moving to Sec.~\ref{secfullynonlinearkerngam},
because the kernel can interpolate the data exactly we can no longer use \eqref{equeddudydddy} to estimate the level of noise $\sigma$.
For a finite-dimensional feature map $\psi$, with data $(X,Y)$, we can decompose $Y=v^T \psi(X)+\sigma Z$ into a signal part $Y_s$ and noise part $Y_s$, s.t. $Y=Y_s+Y_n$.
While $Y_s$ belongs to the linear span of  eigenvectors of $K(X,X)$ associated with non-zero eigenvalues, $Y_n$ also activates the eigenvectors associated with with the null space of $K(X,X)$ and the projection of $Y$ onto that null-space is what allows us to derive $\gamma$ in 
Sec.~\ref{secgamafindimpsi}. 
Since in the interpolatory case, all eigenvalues are strictly positive, we need to choose which eigenvalues are associated with noise differently, as is described in the previous section. With a fixed $\gamma$, we see that if $\lambda_i \gg\gamma$, then $\omega_i\approx 0 $, which contributes in (\ref{eq:explicit_signal_to_noise_ratio}) to yield a low noise-to-signal ratio. Similarly, if $\lambda_i \ll\gamma$, this eigenvalue yields a high noise-to-signal ratio. Thus, we see that the choice of $\gamma$ assigns a noise level to each eigenvalue. While in the finite-dimensional feature map setting, this assignment is binary, here we perform soft thresholding using $\lambda\mapsto \gamma/(\gamma+\lambda)$ to indicate the level of noise of each eigenvalue. 
This interpretation sheds light on the selection of $\gamma$ in equation \eqref{eq:choice_gamma}. Let $\psi$ represent the feature map associated with $K$. Assuming the empirical mean of $\psi(X_i)$ is zero,
the matrix $K(X,X)$ corresponds to an unnormalized kernel covariance matrix $\psi^T(X)\psi(X)$. Consequently, its eigenvalues correspond to $N$ times the variances of the $\psi(X_i)$ across various eigenspaces.
After conducting Ordinary Least Squares regression in the feature space, if the noise variance is estimated as
$\bar{\sigma}^2$, then any eigenspace of the normalized covariance matrix whose eigenvalue is lower than $\bar{\sigma}^2$ cannot be recovered due to the noise. Given this, we set the soft thresholding cutoff to be $\gamma=N\bar{\sigma}^2$ for the unnormalized covariance matrix $K(X,X)$.

\subsection{Z-score/quantile bounds on the noise-to-signal ratio.}\label{secseltrzscore}
If the data is only comprised of noise, then an interval of confidence can be obtained on the noise-to-signal ratio.
To describe this
consider the problem of testing the null hypothesis   ${\bf H_0} : f^\dagger \equiv 0$ (there is no signal) against the alternative hypothesis ${\bf H_1}: f^\dagger \not\equiv 0$ (there is a signal). Under the  null hypothesis   ${\bf H_0}$, the distribution of the noise-to-signal ratio \eqref{eqhejbjebdd} is known and it follows that of the random variable
\begin{equation}
B:= \frac{Z^T D_\gamma^2 Z}{Z^T D_\gamma Z}\,.
\end{equation}
Therefore, the quantiles of $B$ can be used as an interval of confidence on the noise-to-signal ratio if ${\bf H_0} $ is true. More precisely, selecting $\beta$ such that 
$\P[B \leq \beta_\alpha]\approx \alpha$ with $\alpha=0.05$ as a prototypical example, we expect the noise to signal ratio \eqref{eqhejbjebdd} to be, under ${\bf H_0} $, to be larger than $\beta_\alpha$ with probability $\approx 1-\alpha$. The estimation of $\beta$ requires Monte-Carlo sampling.

An alternative approach (in the large data regime) to using the quantile $\beta_\alpha$ is to use the Z-score
\begin{equation}
\Z:=\frac{ \frac{Y^T D_\gamma^2 Y}{Y^T D_\gamma Y}-\E[B]}{\sqrt{\Var[B]}}\,,
\end{equation}
 after estimating $\E[B]$ and $\Var[B]$ via Monte-Carlo sampling.
In particular if ${\bf H_0}$ is  true then $|\Z|\geq z_\alpha$ should occur with probability $\approx \alpha$ with $z_{0.1}=1.65$, $z_{0.05}=1.96$ and $z_{0.01}=2.58$.

\begin{Remark}
Although the quantile $\beta_\alpha$ or the Z-score $\Z$ can be employed to produce an interval of confidence on the noise-to-signal ratio under ${\bf H_0}$ we cannot use them as thresholds for  removing nodes from the list of ancestors as discussed in Sec.~\ref{secstnratioint}
Indeed, observing a  noise-to-signal ratio \eqref{eqhejbjebdd} below the threshold $\beta_\alpha$ does not imply that all the signal has been captured by the kernel; it only implies that some signal has been captured by the kernel $K$. To illustrate this point, consider the setting where one tries to approximate the variable $x_1$ as a function of the variable $x_2$. If $x_1$ is not a function of $x_2$, but of $x_2$ and $x_3$, as in $x_1=\cos(x_2)+\sin(x_3)$, then applying the proposed approach with  $Y$ encoding the values of $x_1$, $X$ encoding the values of $x_2$, and the kernel $K$ depending on $x_2$ could lead to a noise-to-signal ratio below $\beta_\alpha$ due to the presence of a signal in $x_2$. Therefore, although we are missing the variable $x_3$ in the kernel $K$, we would still observe a possibly low noise-to-signal ratio due to the presence of \emph{some} signal in the data.
Summarizing if the data only contains noise then  $\frac{\V(n)}{\V(s)+\V(n)}\geq \beta_\alpha$ should occur with probability $1-\alpha$. 
If the event  $\frac{\V(n)}{\V(s)+\V(n)}< \beta_\alpha$ is observed in the setting of $K=K_{s/t}$=\eqref{eqjehbjedhbedudh} where we try to identify the ancestors of $x_1$, then 
 we can only deduce that $x_3,\ldots,x_d$ contain some signal but perhaps not all of it (we can use this a criterion for pruning $x_2$).
 \end{Remark}

\section{Supplementary information on examples.}\label{secsuplinfoexam}

\subsection{Algebraic equations.}
Although we have used Alg.~\ref{alg2} for the algebraic equations examples presented in Fig.~\ref{fignature4}, Alg.~\ref{alg1} yields the same results with the default signal-to-noise threshold $\tau=0.5$.

\subsection{The chemical reaction network.}
Consider the chemical reaction network example illustrated in  Fig.~\ref{fignature4}.(a).
The proposed mechanism for the hydrogenation of ethylene $(\text{C}_2 \text{H}_4)$ to ethane $(\text{C}_2 \text{H}_6)$, is (writing $[H]$ for the concentration of $H$)  modeled by the following system of differential equations
\begin{equation}\label{eqehyehdydA}
\begin{split}
\frac{d[H_2]}{dt}&=-k_1[H_2]+k_{-1}[H]^2\\
\frac{d[H]}{dt}&=2k_1[H_2]-2k_{-1}[H]^2-k_2 [C_2H_4][H]-k_3[C_2H_5][H]\\
\frac{d[C_2H_4]}{dt}&=-k_2 [C_2H_4][H]\\
\frac{d[C_2H_5]}{dt}&=k_2 [C_2H_4][H]-k_3[C_2H_5][H]
\end{split}
\end{equation}
The primary variables are the concentrations $[H_2]$, $[H]$, $[C_2H_4]$ and $[C_2H_5]$ and their time derivatives $\frac{d[H_2]}{dt}$,
$\frac{d[H]}{dt}$, $\frac{d[C_2H_4]}{dt}$ and $\frac{d[C_2H_5]}{dt}$. The computational hypergraph encodes the functional dependencies \eqref{eqehyehdydA} associated with the chemical reactions. The hyperedges of the hypergraph are assumed to be unknown and the primary variables are assumed to be known. Given $N$ samples from the graph of the form 
\begin{equation}\label{eqfnlsdasiflonkasA}
    \big([H_2](t_i),[H](t_i),[C_2H_4](t_i),[C_2H_5](t_i)\big)_{i=1,\ldots,N}
\end{equation}
our objective is to recover the structure of the hypergraph given by \eqref{eqehyehdydA}, representing the functions by hyperedges.
We create a dataset of the form \eqref{eqfnlsdasiflonkasA} by integrating 50 trajectories of \eqref{eqehyehdydA} for different initial conditions, and each equispaced 50 times from $t = 0$ to $t = 5$.
The dataset is represented in Fig.~\ref{fignature4}.(b)  (the time derivatives of concentrations are estimated by taking the derivatives of the interpolants of those concentrations).  We impose the information that the derivative variables are function of the non-derivative variables to avoid ambiguity in the recovery, as \eqref{eqehyehdydA} is not the unique representation of the functional relation between nodes in the graph.
We implement Alg.~\ref{alg1} with weights $\beta = [0.1,0.01,0.001]$ for linear, quadratic, and nonlinear, respectively (Alg.~\ref{alg2} recovers the same hypergraph). The output graph can be seen in Fig.~\ref{fignature4}.(b). We obtain a perfect recovery of the computational graph and  a correct identification of the relations being quadratic.

\subsection{The Google Covid 19  open data.}
Consider the example illustrated in Fig.~\ref{fignature3}.(e-k).
Categorical data are treated as scalar values, with all variables scaled to achieve a mean of 0 and a variance of 1. We implement three distinct kernel types: linear, quadratic, and Gaussian, with a length scale of 1 for the latter. A weight ratio of $1/10$ is assigned between kernels, signifying that the quadratic kernel is weighted ten times less than the linear kernel. Lastly, the noise parameter, $\gamma$, is determined using the optimal value outlined in Sec.~\ref{secsnrt}.
Initially, a complete graph is constructed using all variables, depicted in  Fig.~\ref{fignature3}.(g). This construction is done using only linear and quadratic kernels. 
The full graph is highly clustered and redundant information is eliminated by selecting representative nodes for each cluster.
 Eliminating redundant nodes is important for two reasons: firstly, it improves the graph's readability, especially with 31 variables; secondly, it avoids hindering graph discovery. In an extreme case, treating two identical variables as distinct would result in one variable’s ancestor simply being its duplicate, yielding an uninformative graph.
Subsequently, the graph discovery algorithm is rerun, with reduced variables due to eliminating redundancy, ushering us into a predominantly noisy regime. With fewer variables available, we use additionally the nonlinear  kernel. Two indicators are employed to navigate our discovery process: the signal-to-noise ratio and the Z-test. The former quantifies the degree to which our regression is influenced by noise, while the latter signals the existence of any signal. We follow the procedure in algorithm \ref{alg2}, resulting in the graph presented in Fig.~\ref{fignature3}.(k).

\subsection{Cell signaling network}\label{seccellsignaling}
Consider the example Fig.~\ref{fignature1}.(l) from \cite{sachs2005causal} and Fig.~\ref{fignature4}.(h-j). 
To identify the ancestors of each node, we apply the algorithm in two stages. First, we learn the dependencies using only linear and quadratic kernels. Fig.~\ref{fignature4}.(h) identifies the resulting graph learned given a subset of $N=2,000$ samples chosen uniformly at random from the dataset. We observe that the graph identified by the algorithm consists of four disconnected clusters where the molecule levels in each cluster are closely related by linear or quadratic dependencies (all connections are linear except for the connection between Akt and PKA, which is quadratic). These edges match a subset of the edges found in the gold standard model identified  in~\cite{sachs2005causal}. With perfect dependencies that have no noise, one can define constraints that reduce the total number of variables in the system. For this noisy dataset that, we treat these dependencies as forming groups of similar variables and introduce a hierarchical approach to learn the connections between groups.
Second, we run the graph discovery algorithm after grouping the molecules into clusters. For each node in the graph, we identified the ancestors of each node by constraining the dependence to be a subset of the clusters. In other words, when identifying the ancestors of a given node $i$ in cluster $C$, the algorithm is only permitted to (1) use ancestors that do not belong to cluster $C$, and (2) include all or none of the variables in each cluster ($j$ in cluster $D\not=C$ is listed as an ancestor if and only if all other nodes $j'$ in cluster $D$ are also listed as ancestors). The ancestors were identified using a Gaussian (fully nonlinear) kernel and the number of by ancestors were selected manually based on the inflection point in the noise-to-signal ratio. The resulting graph is depicted in Fig.~\ref{fignature4}.(i). Each edge is weighted based on its signal-to-noise ratio. We observe that there is a stronger dependence of the Jnk, PKC, and P38 cluster on the PIP3, Plcg, and PIP2 cluster, which closely matches the gold standard model. As compared to approaches based on acyclic DAGs, however, the graph identified by our algorithm also contains feedback loops between the various molecule levels. 
Fig.~\ref{fignature4}.(i-j) displays a side-by-side comparison between  the graph identified with our method and the graph generated in~\cite{sachs2005causal}. To aid in this comparison, we have highlighted different clusters in distinct colors.
We emphasize that while the Bayesian network analysis   in~\cite{sachs2005causal} relied on the control of the sampling of the underlying variables (the simultaneous
measurement of multiple phosphorylated protein and phospholipid components
in thousands of individual primary human immune system cells, and perturbing these
cells with molecular interventions), the reconstruction obtained by our method did not use this information and recovered functional dependencies rather than causal dependencies. 
Interestingly, the information recovered through our method appears to complement and enhance the findings presented in ~\cite{sachs2005causal} (e.g., the linear and noiseless dependencies between variables in the JNK cluster is not something that could easily be inferred from the graph produced in~\cite{sachs2005causal}).

\subsection{BCR reaction network}\label{supp_BCR}

In the high-dimensional example of the BCR reaction network, the computations of terms of the form $y^T k_o(X,X) y$ (i.e., the activations), where $y \in \mathbb{R}^n$ and $k_o(X,X)$ is the $o$-th coordinate of the quadratic kernel ($k(x_i, x_j) = (1 + \langle x_i, x_j \rangle)^2$) becomes the computational bottleneck of our method.  If we let $x_1, \ldots, x_n \in \mathbb{R}^p$ be the points and $x_i^o$ be the $o$-th coordinate of $x_i$, we can compute the activation of the $o-$th coordinate using 
\begin{equation}
k_o(x_i, x_j) = (1 + x^o_i x^o_j)^2 - 1 + 2 x^o_i x^o_j \langle x_i^{-o}, x_j^{-o} \rangle
\end{equation}
where $k_o$ is the $o-$th coordinate of the kernel and $x^{-o}_i$ represents the remaining coordinates of $x_i$. To compute the $n\times n$ kernel matrix of $k_o$ for each $o\in\{1,..,p\}$, we must compute $p\times n\times n$ inner products in $\mathbb{R}^p$, which is a very large computation. Instead, we may use the following reformulation to speed up computations. Notice $\langle x_i, x_j \rangle = x_i^o x_j^o + \langle x_i^{-o}, x_j^{-o} \rangle$, and therefore $k_o(x_i, x_j) = 2 x_i^o x_j^o \langle x_i, x_j \rangle + 2 x_i^o x_j^o - (x_i^o x_j^o)^2$.
Now, define $v^o = (x_i^o y_i)_{i = 1}^p$ and $w^o = ((x_i^o)^2 y_i)_{i =1}^p$, and note that 
\begin{equation}
    y^T K_o y = \sum_{i, j} 2 y_i x_i^o y_j x_j^o (1+ \langle x_i, x_j \rangle) - \sum_{i, j} y_i y_j (x_i^o x_j^o)^2
\end{equation}
and so defining $\tilde{K}= (2 (1+\langle x_i, x_j \rangle))_{i, j = 1}^{n}$ we have that 
\begin{equation}
    y^T K_o y = {v^o}^T \tilde{K} v^o - \left (\sum_{i = 1}^p w_i^o \right )^2 
\end{equation}
Note that $\tilde{K}$ is computed just once for all $p$, and only $v^o$ and $w^o$ change for every ancestor calculation, which is where the main computational gain comes from. 
One may find \href{https://github.com/TheoBourdais/ComputationalHypergraphDiscovery/blob/main/examples/faster_activation_computation.ipynb}{in the GitHub repository of the paper} a comparison of the two methods of computations and observe a tenfold speedup. This speedup is even larger in our implementation of the  BCR example, as GPU acceleration enables the second method to run even faster.

\end{document}